\documentclass[11pt, a4paper, oneside]{article}

\usepackage[margin = 1.3in]{geometry}

\def\citep{\cite}

\usepackage{times}
\usepackage{amsfonts,psfrag,epsfig}
\usepackage{amsfonts,epsfig}
\usepackage{color}
\usepackage{amsmath,amssymb,hyperref,relsize,multirow,caption,capt-of}
\usepackage[amsmath,thmmarks]{ntheorem}
\usepackage{tabularx}
\usepackage{booktabs,floatrow,graphicx} 
\usepackage{blindtext}
\usepackage{floatrow}
\usepackage{tikz}
\usetikzlibrary{intersections}
\usetikzlibrary{shapes,arrows}
\usepackage{subfigure}
\usepackage{floatrow}
\hypersetup{colorlinks=true}
\newtheorem{theorem}{Theorem}
\newtheorem{lemma}[theorem]{Lemma}
\newtheorem{corollary}[theorem]{Corollary}

\newfloatcommand{capbtabbox}{table}[][\FBwidth]
{ \theoremstyle{nonumberplain}\theoremsymbol{\ensuremath{\Box}}
\theorembodyfont{\normalfont}
\newtheorem{proof}{Proof.}
\theoremsymbol{\ensuremath{\Box}}
\theoremheaderfont{\normalfont\itshape}
\theorembodyfont{\normalfont} \theoremstyle{empty}

}

\newcommand{\beq}{\begin{eqnarray}}
\newcommand{\eeq}{\end{eqnarray}}
\newcommand{\beqn}{\begin{equation}}
\newcommand{\eeqn}{\end{equation}}

\newcommand{\E}{\mathbb{E}}

\tikzstyle{decision} = [diamond,
                        draw,
                        fill=blue!20,
                        text width=1.5em,
                        text badly centered,
                        node distance=2cm,
                        inner sep=0pt]
\tikzstyle{block} = [rectangle,
                     draw,
                     fill=blue!20,
                     text centered,
                     rounded corners,
                     minimum height=2em]
\tikzstyle{line} = [draw,
                    -latex']
\tikzstyle{cloud} = [draw,
                     ellipse,
                     fill=red!20,
                     node distance=2cm,
                     minimum height=2em]

\allowdisplaybreaks

\begin{document} 

\title{Simplified Stochastic Feedforward Neural Networks}
\author{Kimin Lee, $\quad$ Jaehyung Kim, $\quad$ Song Chong, $\quad$ Jinwoo Shin
\thanks{K.\ Lee, J.\ Kim, S.\ Chong and J.\ Shin are with School of Electrical Engineering at Korea Advanced Institute of Science Technology, Republic of Korea.
Authors' e-mails: \texttt{kiminlee@kaist.ac.kr, jaehyungkim@kaist.ac.kr songchong@kaist.edu, jinwoos@kaist.ac.kr}}
}


\maketitle

\begin{abstract}
It has been believed that 
stochastic feedforward neural networks (SFNNs) 
have
several advantages beyond deterministic deep neural networks (DNNs):
they have more expressive power
allowing multi-modal mappings and regularize 
better due to their stochastic nature. However, 
training large-scale SFNN is notoriously harder.
In this paper, we aim at developing efficient training methods for SFNN, in particular using
known architectures and pre-trained parameters of DNN. To this end, we propose
a new intermediate stochastic model, called Simplified-SFNN, which can be
built upon any baseline DNN and 
approximates certain SFNN by simplifying 
its upper latent units above stochastic ones.
The main novelty of our approach is in 
establishing the connection between three models, i.e., DNN $\rightarrow$ Simplified-SFNN $\rightarrow$ SFNN, which naturally leads to an efficient training procedure of the stochastic models utilizing
pre-trained parameters of DNN.
Using several popular DNNs, 
we show how they can be effectively transferred to the corresponding stochastic models
for both multi-modal and classification tasks 
on MNIST, TFD, CASIA, CIFAR-10, CIFAR-100 and SVHN datasets. 
In particular, we train a stochastic model 
of 28 layers and 36 million parameters, where 
training such a large-scale stochastic network is significantly challenging without using Simplified-SFNN.
\end{abstract}

\section{Introduction} \label{sec:intro}

Recently, deterministic
deep neural networks (DNNs) 
have demonstrated state-of-the-art performance on many supervised tasks, 
e.g., speech recognition \citep{12speech} and object recognition \citep{12imagenet}.
One of the main components underlying these successes is the efficient training
methods 
for large-scale DNNs, which include
backpropagation \citep{85BackProp},
stochastic gradient descent \citep{51SCGD}, dropout/dropconnect \citep{12dropout,13Dropc}, 
batch/weight normalization \citep{15batch,16WN}, 
and various activation functions \citep{10relu, 16NA}.
On the other hand, stochastic feedforward neural networks (SFNNs) \citep{90RMN} having random latent units
are often necessary in to model the complex stochastic natures of many real-world tasks, e.g.,
structured prediction \citep{13SFNN} and image generation \citep{14GAN}.
Furthermore, it is believed that SFNN has several advantages beyond DNN \citep{15SFNN}:
it has more expressive power for multi-modal learning 
and regularizes better for large-scale networks. 

Training large-scale SFNN is notoriously hard since backpropagation is not directly applicable.
Certain stochastic neural networks using continuous random units are known to 
be trainable efficiently using backpropagation 
with variational techniques and reparameterization tricks \citep{13repar,ruiz2016generalized}.
On the other hand, training SFNN having discrete, i.e., binary or multi-modal, random units is
more difficult since intractable 
probabilistic inference is involved requiring too many random samples.
There have been several efforts toward developing efficient training methods for SFNN having binary random
latent units \citep{90RMN,96Saul,13SFNN,13Bengio, 15SFNN, 15muprop} (see Section \ref{sec:preSFNN} for more details).
However, training a SFNN is still significantly slower than training a DNN of the same architecture,
consequently most prior works have considered a small number (at most 5 or so) of layers in SFNN.
We aim for the same goal, but
our direction is complementary to them. 

Instead of training a SFNN directly, we study whether
pre-trained parameters from a DNN (or easier models) can be transferred to it, possibly with further low-cost fine-tuning.
This approach can be attractive since one can utilize recent advances in DNN design and training.
For example, one can design the network structure of SFNN following known specialized ones of DNN 
and use their pre-trained parameters.
To this end, we first try transferring pre-trained parameters of DNN using
sigmoid activation functions to those of the corresponding SFNN directly. 
In our experiments, the heuristic 
reasonably works well. 
For multi-modal learning, SFNN under such a simple transformation
outperforms DNN. Even for the MNIST classification, the former
performs similarly as the latter (see Section \ref{sec:naive_model} for more details).
However, it is questionable whether a similar strategy works in general, particularly for other unbounded
activation functions like ReLU \citep{10relu} since SFNN has binary, i.e., bounded, random latent units.
Moreover, it loses the regularization benefit of SFNN:
it is believed that transferring parameters of stochastic models to DNN helps its regularization, but the opposite is unlikely.

{\bf Contribution.}  
To address these issues,
we propose a special form of stochastic neural networks, named 
Simplified-SFNN, which is intermediate between SFNN and DNN,
having the following properties.
First, Simplified-SFNN can be
built upon any baseline DNN, possibly having unbounded activation functions. 
The most significant part of our approach lies in 
providing rigorous 
{\em network knowledge transferring} \citep{15net2net} between Simplified-SFNN and DNN.
In particular, we prove that parameters of DNN
can be transformed to those of the corresponding Simplified-SFNN while preserving the performance, i.e.,
both represent the same mapping. 
Second, Simplified-SFNN
approximates certain SFNN, better than DNN, by simplifying 
its upper latent units above stochastic ones using
two different non-linear activation functions.
Simplified-SFNN is much easier to train than SFNN while still maintaining its stochastic
regularization effect.
We also remark that SFNN is a Bayesian network, while
Simplified-SFNN is not. 

The above connection DNN $\rightarrow$ Simplified-SFNN $\rightarrow$ SFNN
naturally suggests the following training procedure for both SFNN and Simplified-SFNN:
train a baseline DNN first and 
then fine-tune its corresponding Simplified-SFNN initialized by the transformed DNN parameters.
The pre-training stage accelerates the training task since DNN is faster to train than Simplified-SFNN.
In addition, one can also utilize known DNN training techniques such as dropout and batch normalization
for fine-tuning Simplified-SFNN.
In our experiments, 
we train SFNN and Simplified-SFNN under the proposed strategy.
They consistently outperform the corresponding DNN for both multi-modal
and classification tasks, 
where the former and the latter are for measuring the model
expressive power and the regularization effect, respectively.
To the best of our knowledge, 
we are the first to confirm that
SFNN indeed regularizes better than DNN.
We also construct the stochastic models following the same network structure of popular
DNNs including Lenet-5 \citep{98MNIST}, NIN \citep{14NIN}, FCN \citep{long2015fully} and WRN \citep{16wideresnet}.
In particular, WRN (wide residual network) 
of 28 layers and 36 million parameters has shown the state-of-art performances
on CIFAR-10 and CIFAR-100 classification datasets, and 
our stochastic models built upon WRN outperform the deterministic WRN 
on the datasets. 

\begin{figure*} [t] \centering
\subfigure[]
{\epsfig{file=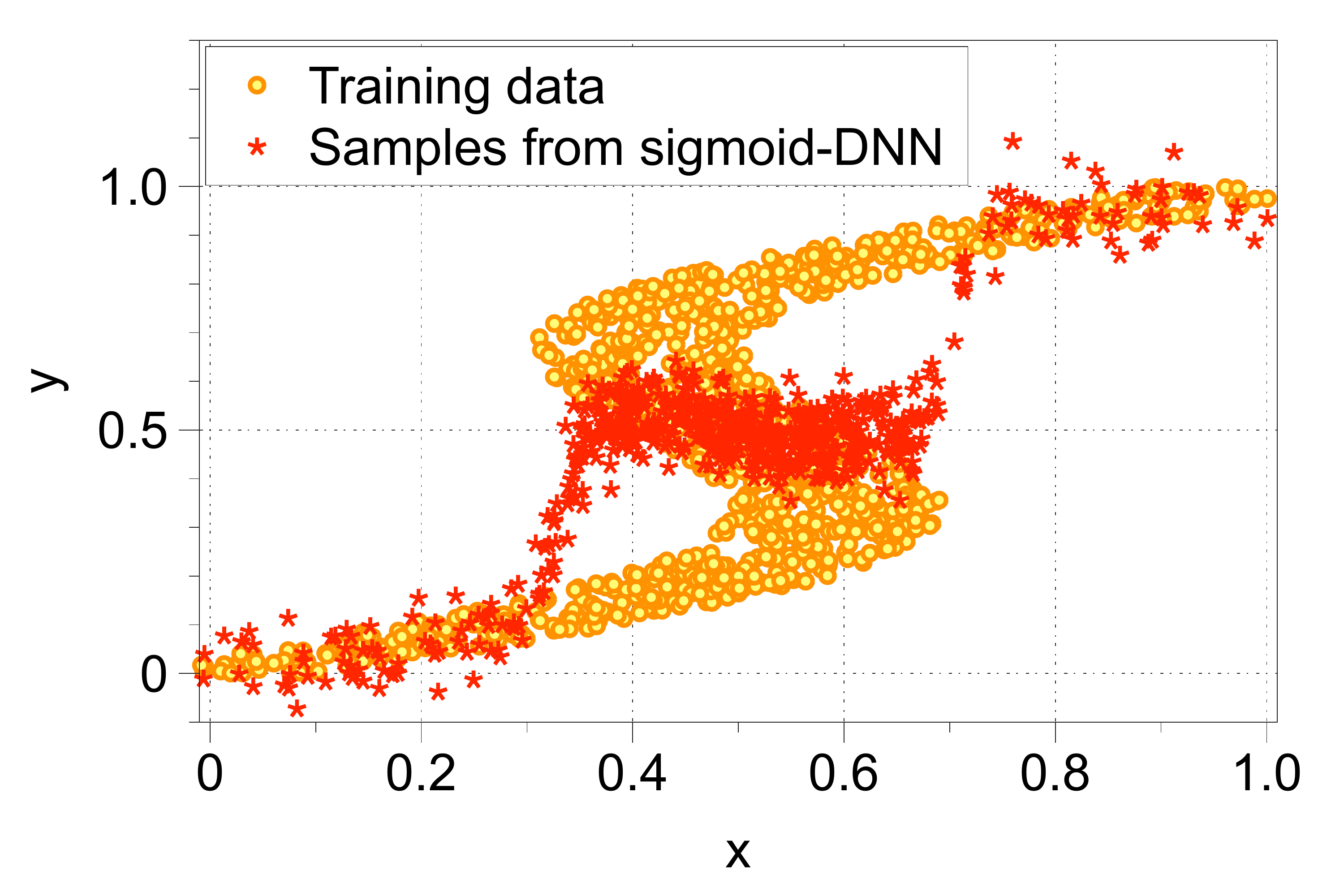,
width=0.45\textwidth}\label{fig:syn_dnn}}
\,
\subfigure[]
{\epsfig{file=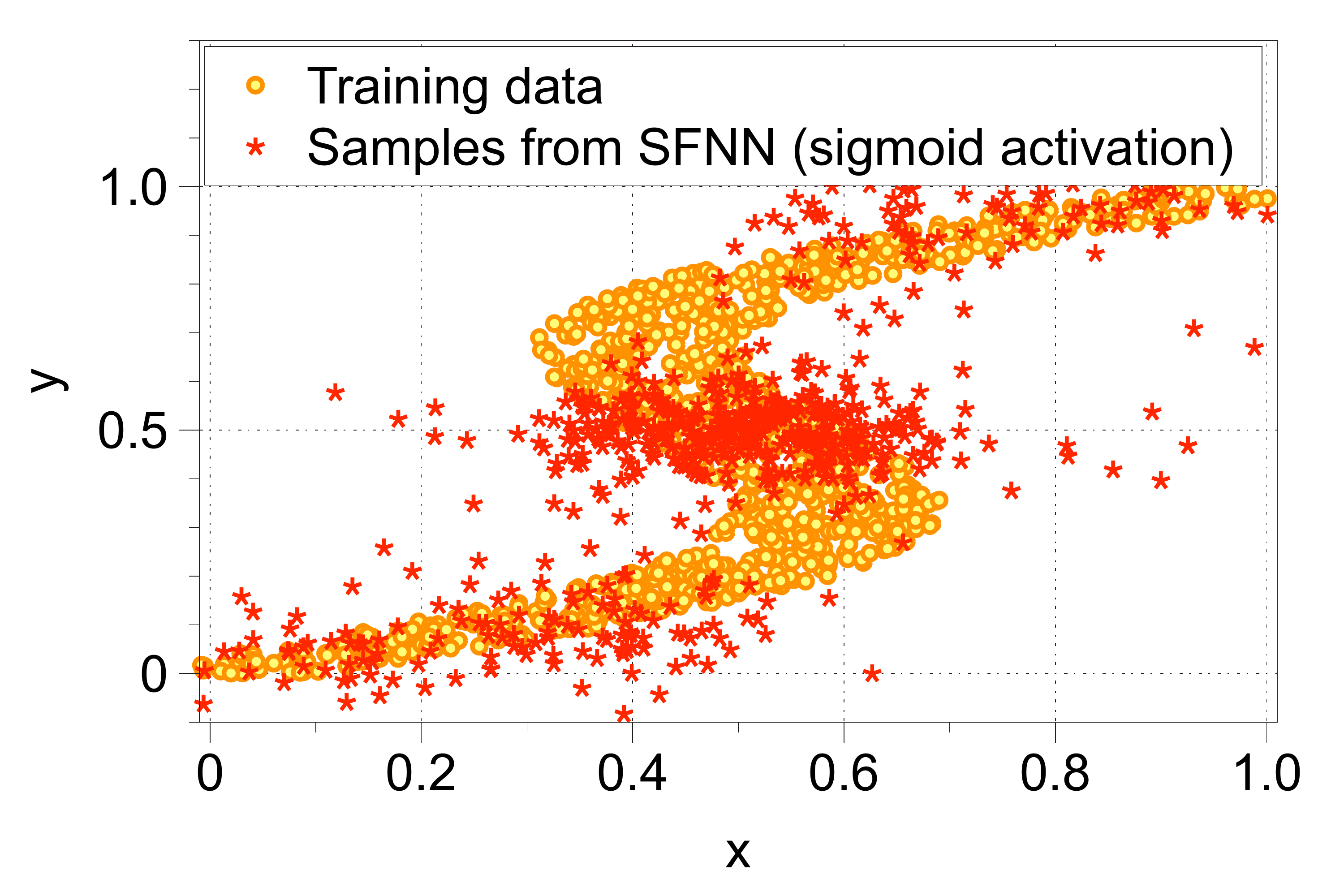,
width=0.45\textwidth}\label{fig:syn_sfnn}}
\caption{The generated samples from (a) sigmoid-DNN and (b) SFNN which uses same parameters trained by sigmoid-DNN.
One can note that SFNN can model the multiple modes in output space $y$ around $x=0.4$.}
\label{fig:syn}
\end{figure*}

\begin{table*}[t] 
\centering
\caption{The performance of simple parameter transformations from DNN to SFNN on the MNIST and synthetic datasets,
where each layer of neural networks 
contains 800 and 50 hidden units for two datasets, respectively. 
For all experiments, only the first hidden layer of the DNN is replaced by stochastic one.
We report negative log-likelihood (NLL) and classification error rates.
}
\label{Tab:naive_sig}
\resizebox{\textwidth}{!}{
\begin{tabular}{@{}lclllllll@{}}
\toprule
\multirow{2}{*}{\begin{tabular}{l}
Inference Model \end{tabular}}
&\multirow{2}{*}{\begin{tabular}{l}
Network Structure \end{tabular}}
&  \multicolumn{3}{c}{MNIST Classification} 
& \multicolumn{1}{c}{Multi-modal Learning} \\
&& \multicolumn{1}{c}{Training NLL} & \multicolumn{1}{c}{Training Error Rate ($\%$)} & \multicolumn{1}{c}{Test Error Rate ($\%$)}& \multicolumn{1}{c}{Test NLL}\\ \midrule
\multicolumn{1}{c}{sigmoid-DNN} & 2 hidden layers &  \multicolumn{1}{c}{0} &  \multicolumn{1}{c}{0}&  \multicolumn{1}{c}{1.54}&  \multicolumn{1}{c}{5.290} \\
\multicolumn{1}{c}{SFNN} & 2 hidden layers &  \multicolumn{1}{c}{0} &  \multicolumn{1}{c}{0}&  \multicolumn{1}{c}{1.56}&  \multicolumn{1}{c}{1.564} \\ \midrule
\multicolumn{1}{c}{sigmoid-DNN} & 3 hidden layers &  \multicolumn{1}{c}{0.002} &  \multicolumn{1}{c}{0.03}&  \multicolumn{1}{c}{1.84}&  \multicolumn{1}{c}{4.880} \\
\multicolumn{1}{c}{SFNN} & 3 hidden layers &  \multicolumn{1}{c}{0.022} &  \multicolumn{1}{c}{0.04}&  \multicolumn{1}{c}{1.81}&  \multicolumn{1}{c}{0.575} \\
\midrule
\multicolumn{1}{c}{sigmoid-DNN} & 4 hidden layers &  \multicolumn{1}{c}{0} &  \multicolumn{1}{c}{0.01}&  \multicolumn{1}{c}{1.74}&  \multicolumn{1}{c}{4.850} \\
\multicolumn{1}{c}{SFNN} & 4 hidden layers &  \multicolumn{1}{c}{0.003} &  \multicolumn{1}{c}{0.03}&  \multicolumn{1}{c}{1.73}&  \multicolumn{1}{c}{0.392} \\ \midrule
\multicolumn{1}{c}{ReLU-DNN} & 2 hidden layers  &  \multicolumn{1}{c}{0.005} &  \multicolumn{1}{c}{0.04}&  \multicolumn{1}{c}{1.49} &  \multicolumn{1}{c}{7.492} \\
\multicolumn{1}{c}{SFNN} & 2 hidden layers &   \multicolumn{1}{c}{0.819} &  \multicolumn{1}{c}{4.50}&  \multicolumn{1}{c}{5.73} &  \multicolumn{1}{c}{2.678} \\ \midrule
\multicolumn{1}{c}{ReLU-DNN} & 3 hidden layers &  \multicolumn{1}{c}{0} &  \multicolumn{1}{c}{0}&  \multicolumn{1}{c}{1.43}&  \multicolumn{1}{c}{7.526} \\
\multicolumn{1}{c}{SFNN} & 3 hidden layers &  \multicolumn{1}{c}{1.174} &  \multicolumn{1}{c}{16.14}&  \multicolumn{1}{c}{17.83}&  \multicolumn{1}{c}{4.468} \\
\midrule
\multicolumn{1}{c}{ReLU-DNN} & 4 hidden layers &  \multicolumn{1}{c}{0} &  \multicolumn{1}{c}{0}&  \multicolumn{1}{c}{1.49}&  \multicolumn{1}{c}{7.572} \\
\multicolumn{1}{c}{SFNN} & 4 hidden layers &  \multicolumn{1}{c}{1.213} &  \multicolumn{1}{c}{13.13}&  \multicolumn{1}{c}{14.64}&  \multicolumn{1}{c}{1.470} \\
\bottomrule
\end{tabular}}
\end{table*}
\section{Simple Transformation from DNN to SFNN} \label{sec:naive_model}

\subsection{Preliminaries for SFNN}\label{sec:preSFNN}
Stochastic feedforward neural network (SFNN) is a hybrid model, which has both stochastic binary and deterministic hidden units.
We first introduce SFNN with one stochastic hidden layer (and without deterministic
hidden layers) for simplicity.
Throughout this paper,
we commonly denote the bias for unit $i$ and 
the weight matrix of the $\ell$-th hidden layer by $b_i^\ell$
and $\mathbf{W}^{\ell}$, respectively.
Then, 
the stochastic hidden layer in SFNN is defined as a binary random vector with $N^1$ units, i.e.,
$\mathbf{h}^1 \in \{0,1\}^{N^1}$, drawn under the following distribution:
\begin{align}
\label{def:SFNN_stochastic}
P\left(\mathbf{h}^1\mid \mathbf{x}\right) = \prod \limits_{i=1}^{N^1} P\left(h_i^1\mid\mathbf{x}\right), 
\text{where} \quad
P\left(h_i^{1}=1\mid \mathbf{x}\right) = \sigma \left( \mathbf{W}^{1}_{i} \mathbf{x}+b_i^{1}\right). 
\end{align}
In the above, $\mathbf x$ is the input vector and $\sigma \left( x \right) =1/\left(1+e^{-x}\right)$ is the sigmoid function.
Our conditional distribution of the output $y$ is defined as follows: 
\begin{align*}
\label{def:SFNN_output}
P\left(y \mid \mathbf{x}\right) 
= \E_{P\left(\mathbf{h}^{1} \mid \mathbf{x}\right)} \left[ P\left( y \mid \mathbf{h}^1 \right) \right] 
= \E_{P\left(\mathbf{h}^{1} \mid \mathbf{x}\right)} \left[\mathcal{N} \left(y \mid \mathbf{W}^2 \mathbf{h}^1 + b^2 ,~{\sigma_y^2} \right) \right],
\end{align*}
where $\mathcal N(\cdot)$ denotes the normal distribution 
with mean $\mathbf{W}^2 \mathbf{h}^1 + b^2$ and (fixed) variance $\sigma_y^2$. 
Therefore, 
$P\left(y \mid \mathbf{x}\right)$ can express a very complex, multi-modal distribution 
since it is a mixture of exponentially many normal distributions.
The multi-layer extension 
is straightforward via a combination of stochastic and deterministic hidden layers 
(see \citep{13SFNN,15SFNN}). 
Furthermore, one can use any other output distributions as like DNN, e.g.,
softmax for classification tasks.

There are two computational issues for training SFNN:
computing expectations with respect to stochastic units in the forward pass and computing gradients in the backward pass. 
One can notice that both are computationally
intractable since they require summations over exponentially many configurations of all stochastic units.
First, in order to handle the issue
in the forward pass, one can use the following Monte Carlo approximation for estimating the expectation:
$P\left(y \mid \mathbf{x}\right) 
\backsimeq \frac{1}{M} \sum \limits_{m=1}^M P(y \mid \mathbf{h}^{(m)}),$ 
where $\mathbf{h}^{(m)} 
\sim P\left(\mathbf{h}^{1}\mid \mathbf{x}\right)$
and $M$ is the number of samples.
This random estimator is unbiased and has relatively low variance \citep{13SFNN} 
since 
one can draw samples from the exact distribution.
Next, in order to handle the issue in backward
pass,
 \citep{90RMN} proposed Gibbs sampling, but it is known that it often
mixes poorly.
 \citep{96Saul} proposed a variational learning based on the mean-field approximation, 
but it has additional parameters making the variational lower bound looser.
More recently, several other techniques have been proposed including
unbiased estimators of the variational bound using importance sampling \citep{13SFNN,15SFNN}
and biased/unbiased estimators of the gradient for approximating backpropagation \citep{13Bengio, 15SFNN, 15muprop}.

\subsection{Simple transformation from sigmoid-DNN and ReLU-DNN to SFNN}

Despite the recent advances, 
training SFNN is still very slow compared to DNN due to the sampling procedures:
in particular, 
it is notoriously hard to train SFNN when the network structure is deeper and wider.
In order to handle these issues, we consider the following approximation:
\begin{align}
\label{naive_approx}
P\left(y \mid \mathbf{x}\right) &=
\E_{P\left(\mathbf{h}^{1} \mid \mathbf{x}\right)} \left[\mathcal{N} \left(y \mid \mathbf{W}^2 \mathbf{h}^1 + b^2, ~\sigma_y^2 \right) \right]\notag\\
&\backsimeq 
 \mathcal{N} \left(y \mid \E_{P\left(\mathbf{h}^{1} \mid \mathbf{x}\right)} \left[\mathbf{W}^2 \mathbf{h}^1\right] + b^2,~\sigma_y^2 \right) \nonumber \\
 &= \mathcal{N} \left(y \mid \mathbf{W}^2 \sigma \left( \mathbf{W}^{1} \mathbf{x}+\mathbf{b}^{1}\right) + b^2,~\sigma_y^2 \right).
\end{align}
Note that the above approximation corresponds to replacing stochastic units by deterministic ones such that their hidden activation values are same as marginal distributions of stochastic units, i.e.,
SFNN can be approximated by DNN using sigmoid activation functions, say sigmoid-DNN.
When there exist more latent layers above the stochastic one,
one has to apply similar approximations to all of them, i.e., exchanging the orders of expectations and non-linear functions, for making the DNN and SFNN are equivalent.
Therefore, instead of training SFNN directly, 
one can try transferring pre-trained
parameters of sigmoid-DNN to those of the corresponding SFNN directly:
train sigmoid-DNN instead of SFNN, and 
replace deterministic units by stochastic ones for the inference purpose.
Although such a strategy looks somewhat `rude', it was often observed in the literature
that it reasonably works well for SFNN \citep{15SFNN} and we also evaluate it as reported in Table \ref{Tab:naive_sig}. 
We also note that similar approximations appear in the context of dropout:
it trains a stochastic model averaging exponentially many DNNs sharing parameters,
but also approximates a single DNN well.

Now we investigate a similar transformation in the case
when the DNN uses the unbounded ReLU activation function, say ReLU-DNN.
Many recent deep networks are of the ReLU-DNN type because they mitigate the gradient vanishing problem, and their pre-trained parameters are often available.
Although it is straightforward to build SFNN from sigmoid-DNN, 
it is less clear in this case since ReLU is unbounded.
To handle this issue, we 
redefine the stochastic latent units of SFNN: 
\begin{align}
\label{def:relu}
P\left(\mathbf{h}^1 \mid \mathbf{x}\right) = \prod \limits_{i=1}^{N^1} P\left(h_i^1\mid \mathbf{x}\right), 
\text{where}\quad
P\left(h_i^{1}=1\mid \mathbf{x}\right) = \min\bigg\{\alpha f\left(\mathbf{W}^{1}_{i} \mathbf{x}+b_i^{1}\right) ,1\bigg\}. 
\end{align}
In the above,
$f(x) = \max\{x,0\}$ is the ReLU activation function
and $\alpha$ is some hyper-parameter.
A simple transformation can be defined similarly as the case of sigmoid-DNN via replacing deterministic units by stochastic ones.
However, to preserve the parameter information of ReLU-DNN,
one has to choose $\alpha$ such that 
$\alpha f\left(\mathbf{W}^{1}_{i} \mathbf{x}+b_i^{1}\right)\leq 1$
and rescale upper parameters $\mathbf{W}^{2}$ as follows:
\begin{align} \label{eq:relutransform}
\alpha^{-1} \leftarrow \underset{i,\mathbf{x}}{{\max}} ~\left|f\left(\mathbf{\widehat W}^{1}_{i} \mathbf{x} +{\widehat b}_i^{1}\right)\right|,~ 
 \left( \mathbf{W}^{1},~\mathbf{b}^{1} \right)
 \leftarrow  \left( 
\mathbf{\widehat{W}}^{1},~
\mathbf{\widehat{b}}^{1}   \right), ~
 \left(  \mathbf{W}^{2},~\mathbf{b}^{2} \right)
\leftarrow  
\left(  ~
\mathbf{\widehat{W}}^{2}/\alpha ,~
 \mathbf{\widehat{b}}^{2}  \right). 
\end{align}
Then, applying similar approximations as in \eqref{naive_approx}, i.e.,
exchanging the orders of expectations and non-linear functions,
one can observe that ReLU-DNN and SFNN are equivalent.

We evaluate the performance of the simple transformations 
from DNN to SFNN on the MNIST dataset \citep{98MNIST} and the synthetic dataset \citep{94MDN},
where the former and the latter are popular datasets used for
a classification task and a multi-modal (i.e.,
one-to-many mappings) prediction learning, respectively.
In all experiments reported in this paper, we
commonly use
the softmax and Gaussian with standard deviation of $\sigma_y=0.05$
are used 
for the output probability on classification and regression
tasks, respectively.
The only first hidden layer of DNN is replaced by a stochastic layer, 
and we use 500 samples for estimating the expectations in the SFNN inference.
As reported in Table \ref{Tab:naive_sig},
we observe that the simple transformation often works well for both tasks:
the SFNN and sigmoid-DNN inferences (using same parameters trained by sigmoid-DNN) 
perform similarly for the classification task
and the former significantly outperforms the latter for the multi-modal task (also see Figure \ref{fig:syn}).
It might suggest some possibilities 
that the expensive SFNN training might not be not necessary, depending on
the targeted learning quality.
However,
in case of ReLU, SFNN performs much worse than ReLU-DNN for the MNIST classification task
under the parameter transformation.

\begin{figure*} [t] \centering
\subfigure[]
{\epsfig{file=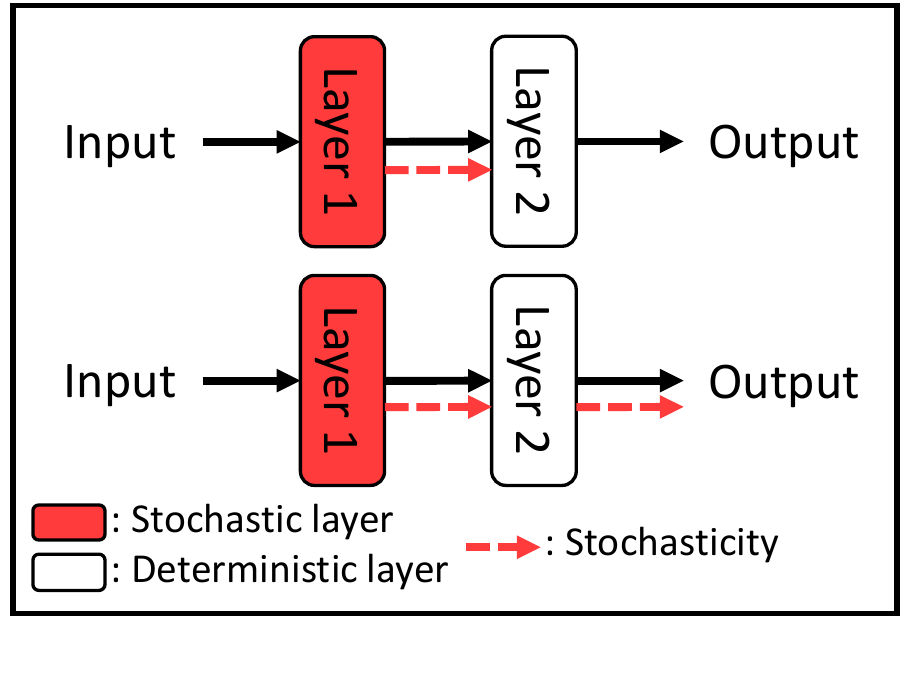, width=0.32\textwidth}\label{fig:fig1b}}
\,
\subfigure[]
{\epsfig{file=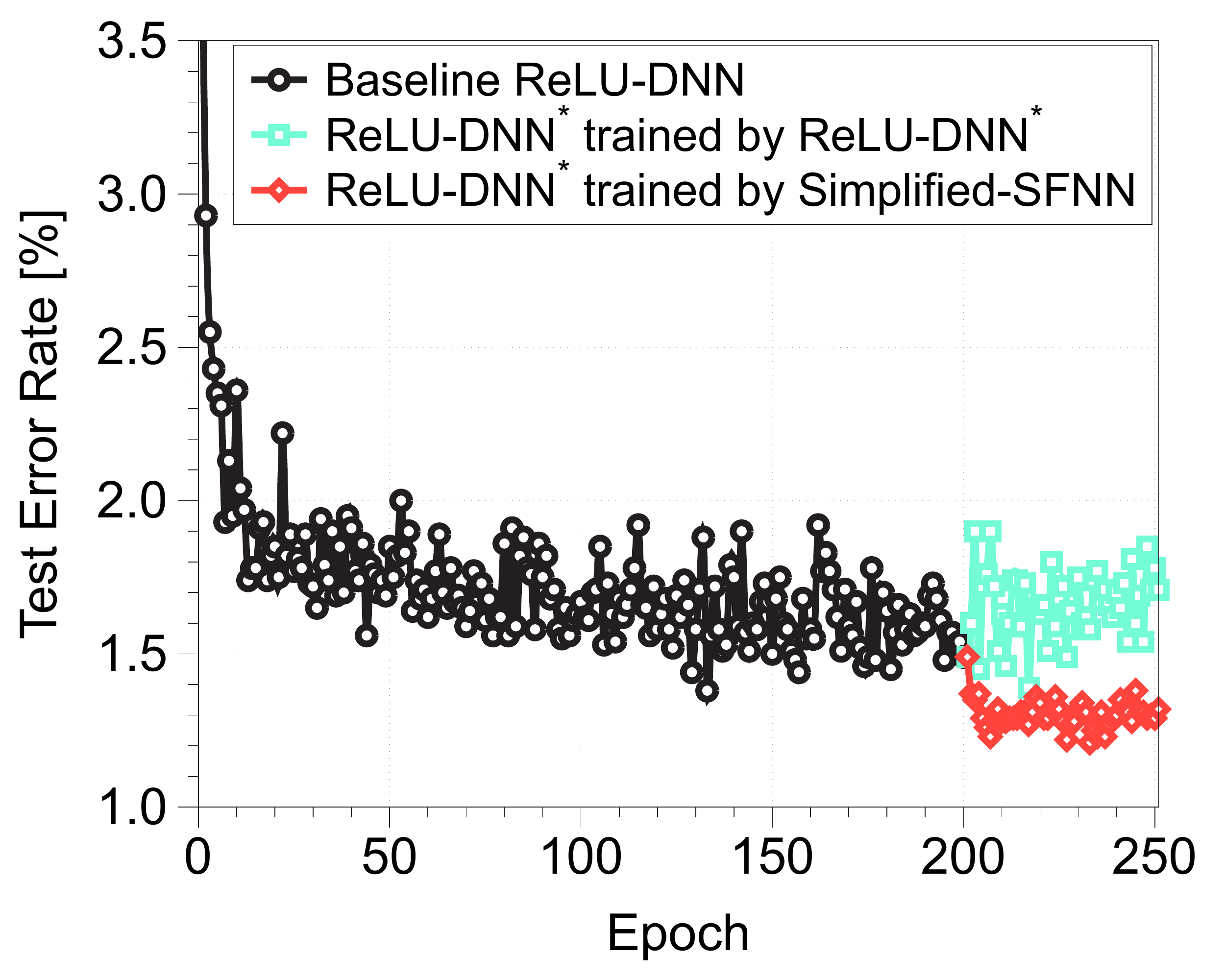, width=0.32\textwidth}\label{fig:fig1a}}
\,
\subfigure[]
{\epsfig{file=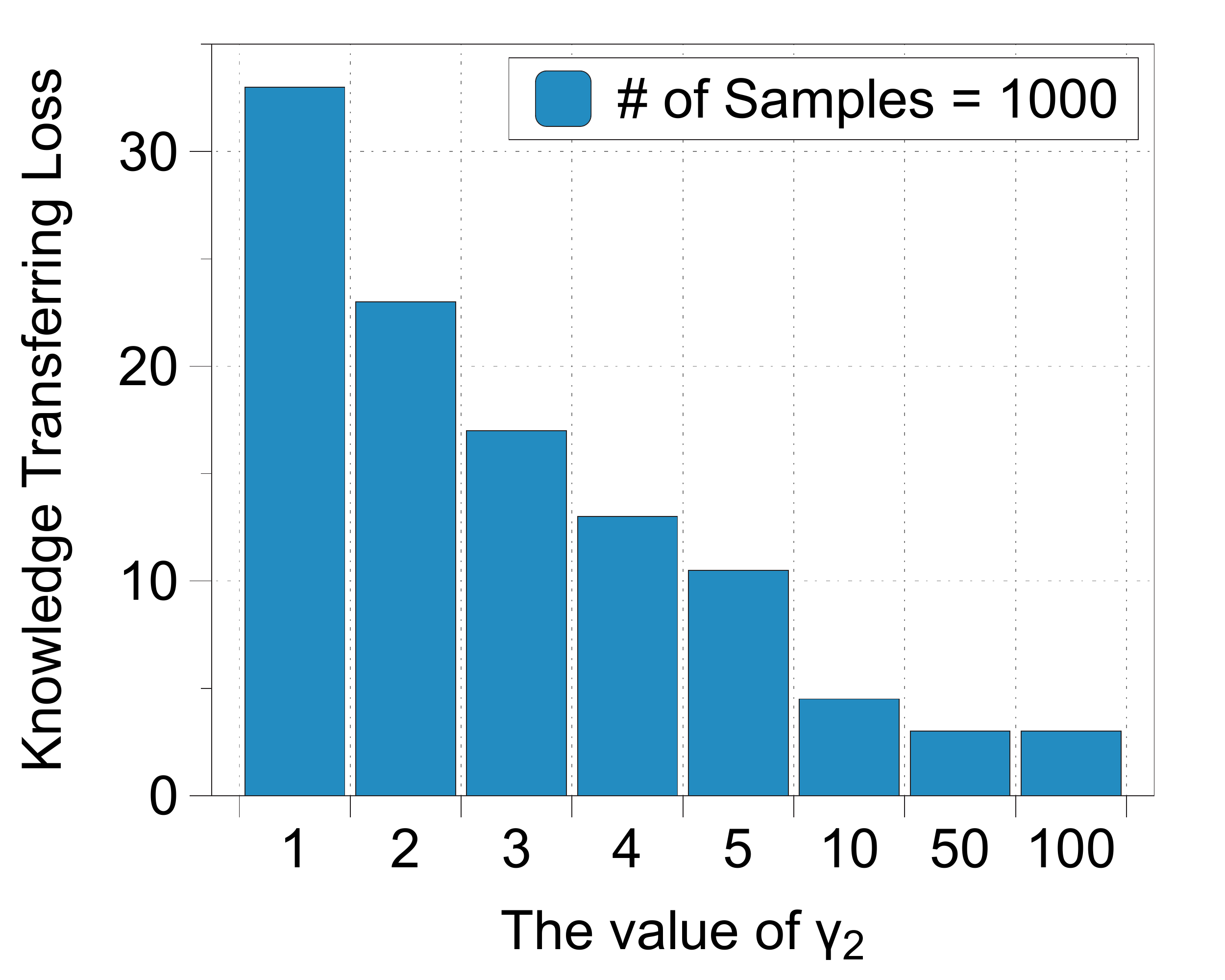, width=0.32\textwidth}\label{fig:fig1c}}
\caption{ (a) Simplified-SFNN (top) and SFNN (bottom). 
The randomness of the stochastic layer propagates only  to its upper layer in the case of Simplified-SFNN.
(b) For first 200 epochs, we train a baseline ReLU-DNN.
  Then, we train simplified-SFNN initialized by the DNN parameters under transformation
  \eqref{eq:weighttrans} with $\gamma_2=50$. 
We observe that training ReLU-DNN$^*$ directly does not 
reduce the test error
even when network knowledge transferring still holds between the baseline ReLU-DNN and the corresponding ReLU-DNN$^*$.
(c) As the value of $\gamma_2$ increases, knowledge transferring loss measured as
$\frac{1}{|D|} \frac{1}{N^{\ell}} \sum \limits_{\mathbf{x}} \sum \limits_i \left| h_i^{\ell}\left(\mathbf{x}\right) 
-\widehat{h}_i^{\ell}\left(\mathbf{x}\right) \right|$ is decreasing.
  }
    \label{fig_comp}
\end{figure*}

\begin{table*}[t] 
\centering
\caption{Classification test error rates [$\%$] on MNIST, where
each layer of neural networks contains 800 hidden units. 
All Simplified-SFNNs are constructed by replacing the first hidden layer of a baseline DNN with stochastic hidden layer. We also consider training DNN and fine-tuning Simplified-SFNN using batch normalization (BN)
and dropout (DO). The performance improvements beyond baseline DNN (due to fine-tuning DNN parameters under  Simplified-SFNN) are calculated in the bracket.}
\label{Tab:relu_SSFNN}
\resizebox{\textwidth}{!}{
\begin{tabular}{@{}clccllllll@{}}
\toprule
\multirow{1}{*}{\begin{tabular}{l}
Inference Model \end{tabular}}
&\multirow{1}{*}{\begin{tabular}{l}
Training Model \end{tabular}}
&\multirow{1}{*}{\begin{tabular}{l}
Network Structure \end{tabular}}
&  \multicolumn{1}{c}{without BN \& DO} & \multicolumn{1}{c}{with BN} & \multicolumn{1}{c}{with DO} \\ \midrule
\multicolumn{1}{c}{sigmoid-DNN} & sigmoid-DNN & 2 hidden layers 
&  \multicolumn{1}{c}{1.54} &  \multicolumn{1}{c}{1.57}&  \multicolumn{1}{c}{1.25} \\
\multicolumn{1}{c}{SFNN} & sigmoid-DNN & 2 hidden layers 
&  \multicolumn{1}{c}{1.56} &  \multicolumn{1}{c}{2.23}&  \multicolumn{1}{c}{1.27} \\
\multicolumn{1}{c}{Simplified-SFNN} & fine-tuned by Simplified-SFNN & 2 hidden layers &  \multicolumn{1}{c}{1.51} &  \multicolumn{1}{c}{1.5}&  \multicolumn{1}{c}{1.11} \\
\multicolumn{1}{c}{sigmoid-DNN$^*$} & fine-tuned by Simplified-SFNN & 2 hidden layers 
&  \multicolumn{1}{c}{1.48 (0.06)} &  \multicolumn{1}{c}{1.48 (0.09)}&  \multicolumn{1}{c}{1.14 (0.11)} \\ 
\multicolumn{1}{c}{SFNN} & fine-tuned by Simplified-SFNN & 2 hidden layers 
&  \multicolumn{1}{c}{1.51} &  \multicolumn{1}{c}{1.57}&  \multicolumn{1}{c}{1.11} \\\midrule
\multicolumn{1}{c}{ReLU-DNN} & ReLU-DNN & 2 hidden layers 
&  \multicolumn{1}{c}{1.49} &  \multicolumn{1}{c}{1.25}&  \multicolumn{1}{c}{1.12} \\
\multicolumn{1}{c}{SFNN} & ReLU-DNN & 2 hidden layers 
&  \multicolumn{1}{c}{5.73} &  \multicolumn{1}{c}{3.47}&  \multicolumn{1}{c}{1.74} \\
\multicolumn{1}{c}{Simplified-SFNN} & fine-tuned by Simplified-SFNN & 2 hidden layers &  \multicolumn{1}{c}{1.41} &  \multicolumn{1}{c}{1.17}&  \multicolumn{1}{c}{1.06} \\
\multicolumn{1}{c}{ReLU-DNN$^*$} & fine-tuned by Simplified-SFNN & 2 hidden layers 
&  \multicolumn{1}{c}{1.32 (0.17)} &  \multicolumn{1}{c}{1.16 (0.09)}&  \multicolumn{1}{c}{1.05 (0.07)} \\ 
\multicolumn{1}{c}{SFNN} & fine-tuned by Simplified-SFNN & 2 hidden layers 
&  \multicolumn{1}{c}{2.63} &  \multicolumn{1}{c}{1.34}&  \multicolumn{1}{c}{1.51} \\ \midrule
\multicolumn{1}{c}{ReLU-DNN} & ReLU-DNN & 3 hidden layers 
&  \multicolumn{1}{c}{1.43} &  \multicolumn{1}{c}{1.34}&  \multicolumn{1}{c}{1.24} \\
\multicolumn{1}{c}{SFNN} & ReLU-DNN & 3 hidden layers 
&  \multicolumn{1}{c}{17.83} &  \multicolumn{1}{c}{4.15}&  \multicolumn{1}{c}{1.49} \\
\multicolumn{1}{c}{Simplified-SFNN} & fine-tuned by Simplified-SFNN & 3 hidden layers 
&  \multicolumn{1}{c}{1.28} &  \multicolumn{1}{c}{1.25}&  \multicolumn{1}{c}{1.04} \\
\multicolumn{1}{c}{ReLU-DNN$^*$} & fine-tuned by Simplified-SFNN & 3 hidden layers 
&  \multicolumn{1}{c}{1.27 (0.16)} &  \multicolumn{1}{c}{1.24 (0.1)}&  \multicolumn{1}{c}{{\bf 1.03 (0.21)}} \\
\multicolumn{1}{c}{SFNN} & fine-tuned by Simplified-SFNN & 3 hidden layers 
&  \multicolumn{1}{c}{1.56} &  \multicolumn{1}{c}{1.82}&  \multicolumn{1}{c}{1.16} \\\midrule
\multicolumn{1}{c}{ReLU-DNN} & ReLU-DNN & 4 hidden layers 
&  \multicolumn{1}{c}{1.49} &  \multicolumn{1}{c}{1.34}&  \multicolumn{1}{c}{1.30} \\
\multicolumn{1}{c}{SFNN} & ReLU-DNN & 4 hidden layers 
&  \multicolumn{1}{c}{14.64} &  \multicolumn{1}{c}{3.85}&  \multicolumn{1}{c}{2.17} \\
\multicolumn{1}{c}{Simplified-SFNN} & fine-tuned by Simplified-SFNN & 4 hidden layers &  \multicolumn{1}{c}{1.32} &  \multicolumn{1}{c}{1.32}&  \multicolumn{1}{c}{1.25} \\
\multicolumn{1}{c}{ReLU-DNN$^*$} & fine-tuned by Simplified-SFNN & 4 hidden layers 
&  \multicolumn{1}{c}{1.29 (0.2)} &  \multicolumn{1}{c}{1.29 (0.05)}&  \multicolumn{1}{c}{1.25 (0.05)} \\
\multicolumn{1}{c}{SFNN} & fine-tuned by Simplified-SFNN & 4 hidden layers 
&  \multicolumn{1}{c}{3.44} &  \multicolumn{1}{c}{1.89}&  \multicolumn{1}{c}{1.56} \\
\bottomrule \end{tabular}}
\end{table*}

\section{Transformation from DNN to SFNN via Simplified-SFNN}
\label{sec:model}

In this section, we propose an advanced method to utilize the pre-trained parameters of DNN for training SFNN.
As shown in the previous section, simple parameter transformations from DNN to SFNN do not clearly work in general,
in particular for activation functions other than sigmoid. 
Moreover, training
DNN does not utilize the stochastic regularizing effect, which is an important benefit of SFNN. 
To address the issues, we design an intermediate model, called Simplified-SFNN.
The proposed model is a special form of stochastic neural network, 
which approximates certain SFNN 
by simplifying its upper latent units above stochastic ones.
Then, we establish more rigorous connections between three models: DNN $\rightarrow$ Simplified-SFNN $\rightarrow$ SFNN, which leads 
to an efficient training procedure of the stochastic models utilizing pre-trained parameters of DNN.
In our experiments, we evaluate the strategy for various tasks and popular DNN architectures.

\subsection{Simplified-SFNN of two hidden layers and non-negative activation functions}

For clarity of presentation, we first introduce Simplified-SFNN with two hidden layers and non-negative activation functions,
where its extensions to multiple layers and general activation functions
are presented in Section \ref{sec:mainresult}.
We also remark that we primarily describe fully-connected Simplified-SFNNs, but their convolutional versions can also be naturally defined.
In Simplified-SFNN of two hidden layers,
we assume that the first and second hidden layers consist of stochastic binary hidden units 
and deterministic ones, respectively.
As in \eqref{def:relu},
the first layer is defined as a binary random vector with $N^1$ units, i.e., 
$\mathbf{h}^1 \in \{0,1\}^{N^1}$, drawn under the following distribution: 
\begin{align}
\label{def:pro}
P\left(\mathbf{h}^1\mid \mathbf{x}\right) = \prod \limits_{i=1}^{N^1} P\left(h_i^1\mid \mathbf{x}\right), 
\text{where}\quad
P\left(h_i^{1}=1\mid \mathbf{x}\right) = \min\bigg\{\alpha_1 f\left(\mathbf{W}^{1}_{i} \mathbf{x}+b_i^{1}\right) ,1\bigg\}.
\end{align}
where $\mathbf x$ is the input vector,
$\alpha_1>0$ is a hyper-parameter for the first layer,
and $f:\mathbb{R}\rightarrow \mathbb{R}_{+}$ is some non-negative non-linear activation function
with $\left| f^\prime(x) \right| \leq 1$ for all $x\in \mathbb R$, 
e.g., ReLU and sigmoid activation functions.
Now the second layer
is defined as the following deterministic vector with $N^2$ units, i.e., $\mathbf{h}^{2}(\mathbf{x})\in \mathbb{R}^{N^{2}}$: 
\begin{align}
\label{def:deter}
 \mathbf{h}^2\left(\mathbf{x}\right)  
 =\left[f \left( \alpha_2 \left(   \E_{P \left(\mathbf{h}^{1}\mid \mathbf{x}\right)} \left[s\left(\mathbf{W}^{2}_{j}\mathbf{h}^{1}+b_j^{2}\right)\right]   - s\left(0\right)  \right) \right):\forall j  \right], 
\end{align}
where $\alpha_2>0$ is a hyper-parameter for the second layer and $s:\mathbb{R}\rightarrow \mathbb{R}$ is 
a differentiable function with $|s^{\prime\prime}(x)| \leq 1$ for all $x\in \mathbb R$, e.g., sigmoid and tanh 
functions.
In our experiments, we use the sigmoid function for $s(x)$.
Here, one can note that the proposed model also has the same computational issues with SFNN in forward and backward
passes due to the complex expectation.
One can train Simplified-SFNN similarly as SFNN: 
we use Monte Carlo approximation for estimating the expectation and 
the (biased) estimator of the gradient for approximating backpropagation inspired by  \cite{15SFNN} (see Section \ref{sec:train} for more details).

We are interested in transferring parameters of DNN to Simplified-SFNN 
to utilize the training benefits of DNN since the former is much faster to train than the latter.
To this end, we consider the following DNN of which
$\ell$-th hidden layer is deterministic and defined 
as follows:
\begin{align} \label{ea:ENN}
&\mathbf{\widehat{h}}^\ell \left(\mathbf{x}\right) = \left[~ \widehat{h}_i^\ell \left(\mathbf{x}\right)= f \left( \mathbf{\widehat W}^\ell_{i}\mathbf{\widehat{h}}^{\ell-1}\left(\mathbf{x}\right) + {\widehat  b}_i^\ell \right):\forall i~ \right], 
\end{align}
where $\mathbf{\widehat h}^0 (\mathbf{x})=\mathbf{x}$.
As stated in the following theorem, 
we establish a rigorous way
how to initialize parameters of Simplified-SFNN 
in order to transfer the knowledge stored in DNN. 
\begin{theorem} \label{thm:2samefeature}
Assume that both DNN and Simplified-SFNN with two hidden layers have same network structure with non-negative activation function $f$.
Given parameters $\{ \mathbf{\widehat W}^\ell,\mathbf{\widehat b}^\ell: {\ell = 1,2} \}$ of DNN and input dataset $D$, 
choose those of Simplified-SFNN as follows:
\begin{align} \label{eq:weighttrans}
\left( \alpha_1, \mathbf{W}^{1},~\mathbf{b}^{1} \right)
\leftarrow  \left( \frac{1}{\gamma_1},~
\mathbf{\widehat{W}}^{1},~
\mathbf{\widehat{b}}^{1}   \right), ~ 
\left( \alpha_2, \mathbf{W}^{2},~\mathbf{b}^{2} \right)
\leftarrow  
\left(  \frac{\gamma_2\gamma_1}{ s^{\prime} \left(0\right)}, ~
\frac{1}{\gamma_2} \mathbf{\widehat{W}}^{2},~
\frac{1}{\gamma_1\gamma_2 } \mathbf{\widehat{b}}^{2}  \right),
\end{align}
where $\gamma_1= \underset{i,\mathbf{x}\in D}{{\max}} ~\left|f\left(\mathbf{\widehat W}^{1}_{i} \mathbf{x} +{\widehat b}_i^{1}\right)\right|$ and $\gamma_2>0$ is any positive constant.
Then, for all $j,\mathbf{x}\in D$, it follows that
\begin{align*}
\left| h_j^2\left(\mathbf{x}\right) 
-\widehat{h}_j^2\left(\mathbf{x}\right) \right| 
 \leq  
\frac{\gamma_1\left(\sum_i \left| {\widehat W}_{ij}^2\right|+ {\widehat b}_{j}^2\gamma_1^{-1} \right)^2}{2s' \left(0\right)\gamma_2}.
\quad 
\end{align*}
\end{theorem}
The proof of the above theorem is presented in Section \ref{sec:proofVCSig}.
Our proof  is built upon the first-order Taylor expansion of non-linear function $s(x)$.
Theorem \ref{thm:2samefeature} implies that 
one can make Simplified-SFNN represent the function values of DNN with bounded errors using a linear transformation.
Furthermore, the errors can be made arbitrarily small
by choosing large $\gamma_2$, i.e., 
$\lim \limits_{\gamma_2 \to \infty} \left| h_j^2\left(\mathbf{x}\right) 
-\widehat{h}_j^2\left(\mathbf{x}\right) \right| = 0,$ $\forall j,\mathbf{x}\in D.$
Figure \ref{fig:fig1c} shows that knowledge transferring loss decreases 
as $\gamma_2$ increases on MNIST classification. 
Based on this, we choose $\gamma_2=50$ commonly for all experiments. 

\subsection{Why Simplified-SFNN ?} \label{sec:why}


Given a Simplified-SFNN model, the corresponding SFNN can be naturally defined by taking out 
the expectation in \eqref{def:deter}.
As illustrated in Figure \ref{fig:fig1b},
the main difference between SFNN and Simplified-SFNN is that
the randomness of the stochastic layer propagates only to its upper layer in the latter, i.e.,
the randomness of $\mathbf{h}^1$ is averaged out at
its upper units $\mathbf{h}^2$
and does not propagate to $\mathbf{h}^3$ or output $y$.
Hence, Simplified-SFNN is no longer a Bayesian network.
This makes training Simplified-SFNN much easier than SFNN since 
random samples are not required at some layers\footnote{
For example, if one replaces the {first feature maps in the fifth residual unit} of Pre-ResNet having 164 layers \citep{16resnet}
by stochastic ones, then 
the corresponding DNN, Simplified-SFNN and SFNN took 
1 mins 35 secs, 2 mins 52 secs and 16 mins 26 secs per each training epoch, respectively, on our machine
with one Intel CPU (Core i7-5820K 6-Core@3.3GHz) and one NVIDIA GPU (GTX Titan X, 3072 CUDA cores).
Here, we trained both stochastic models 
using the biased estimator \citep{15SFNN} with 10 random samples on CIFAR-10 dataset.}
and consequently the quality of gradient estimations can also be improved, in particular
for unbounded activation functions.
Furthermore, one can use the same approximation procedure \eqref{naive_approx}
to see that Simplified-SFNN approximates SFNN.
However, since Simplified-SFNN still maintains binary random units, 
it uses approximation steps later, in comparison with DNN. 
In summary, Simplified-SFNN is an intermediate model between DNN and SFNN, i.e., DNN $\rightarrow$ Simplified-SFNN $\rightarrow$ SFNN.

\begin{table*}[t] 
\centering
\caption{Test negative log-likelihood (NLL) on MNIST and TFD datasets, where each layer of neural networks contains 200 hidden units. All Simplified-SFNNs are constructed by replacing the first hidden layer of a baseline DNN with stochastic hidden layer.} 
\label{Tab:mult_exp2}
\resizebox{\textwidth}{!}{
\begin{tabular}{@{}ccclclclll@{}}
\toprule
\multirow{2}{*}{\begin{tabular}[c]{@{}c@{}} Inference Model \end{tabular}} & \multirow{2}{*}{\begin{tabular}[c]{@{}c@{}}      Training Model \end{tabular}}
 & \multicolumn{2}{c}{\begin{tabular}[c]{@{}c@{}} MNIST \end{tabular}} &  \multicolumn{2}{c}{\begin{tabular}[c]{@{}c@{}} TFD \end{tabular}} \\ 
& \multicolumn{1}{c}{} & \multicolumn{1}{c}{2 Hidden Layers}& \multicolumn{1}{c}{3 Hidden Layers} & \multicolumn{1}{c}{2 Hidden Layers}& \multicolumn{1}{c}{3 Hidden Layers} \\ \midrule
\multicolumn{1}{c}{sigmoid-DNN}
& \multicolumn{1}{c}{sigmoid-DNN} & \multicolumn{1}{c}{1.409} &  \multicolumn{1}{c}{1.720} &  \multicolumn{1}{c}{-0.064} &  \multicolumn{1}{c}{0.005}\\
\multicolumn{1}{c}{SFNN}
 & \multicolumn{1}{c}{sigmoid-DNN} &  \multicolumn{1}{c}{0.644} &  \multicolumn{1}{c}{1.076} &   \multicolumn{1}{c}{-0.461} &  \multicolumn{1}{c}{-0.401} \\
 \multicolumn{1}{c}{\begin{tabular}[c]{@{}c@{}} Simplified-SFNN \end{tabular}}
& \multicolumn{1}{c}{fine-tuned by Simplified-SFNN} &  \multicolumn{1}{c}{1.474} &  \multicolumn{1}{c}{1.757}  &  \multicolumn{1}{c}{-0.071} &  \multicolumn{1}{c}{-0.028} \\
\multicolumn{1}{c}{\begin{tabular}[c]{@{}c@{}} SFNN \end{tabular}}
& \multicolumn{1}{c}{fine-tuned by Simplified-SFNN} &  \multicolumn{1}{c}{0.619} &  \multicolumn{1}{c}{0.991} &  \multicolumn{1}{c}{{\bf -0.509}} &  \multicolumn{1}{c}{-0.423} \\ \midrule
\multicolumn{1}{c}{ReLU-DNN}
& \multicolumn{1}{c}{ReLU-DNN}  &  \multicolumn{1}{c}{1.747} &  \multicolumn{1}{c}{1.741}  &  \multicolumn{1}{c}{1.271} &  \multicolumn{1}{c}{1.232} &  \multicolumn{1}{c}{} \\
\multicolumn{1}{c}{SFNN}
& \multicolumn{1}{c}{ReLU-DNN}  &  \multicolumn{1}{c}{-1.019} &  \multicolumn{1}{c}{-1.021}   & \multicolumn{1}{c}{0.823} &  \multicolumn{1}{c}{1.121} &  \multicolumn{1}{c}{}\\
\multicolumn{1}{c}{\begin{tabular}[c]{@{}c@{}} Simplified-SFNN \end{tabular}}
& \multicolumn{1}{c}{fine-tuned by Simplified-SFNN}  &  \multicolumn{1}{c}{2.122} &  \multicolumn{1}{c}{2.226} &  \multicolumn{1}{c}{0.175} &  \multicolumn{1}{c}{0.343} &  \multicolumn{1}{c}{} \\ 
\multicolumn{1}{c}{\begin{tabular}[c]{@{}c@{}} SFNN \end{tabular}}
&  \multicolumn{1}{c}{fine-tuned by Simplified-SFNN} &  \multicolumn{1}{c}{{\bf -1.290}} &  \multicolumn{1}{c}{-1.061} & \multicolumn{1}{c}{-0.380} &  \multicolumn{1}{c}{-0.193} &  \multicolumn{1}{c}{} \\ 
\bottomrule
\end{tabular}}
\end{table*}

The above connection naturally suggests the following training procedure
for both SFNN and Simplified-SFNN: train a baseline DNN first and then fine-tune its corresponding Simplified-SFNN initialized by the transformed DNN parameters.
Finally, the fine-tuned parameters can be used for SFNN as well.
We evaluate the strategy for the MNIST classification.
The MNIST dataset consists of $28\times28$ pixel greyscale images, each containing a digit 0 to 9 with 60,000 training and 10,000 test images. 
For this experiment,
we do not use any data augmentation or pre-processing. 
The loss was minimized using ADAM learning rule \citep{15ADAM} with a mini-batch size of 128. We used an exponentially decaying learning rate.
Hyper-parameters are tuned on the validation set consisting of the last 10,000 training images.
All Simplified-SFNNs are constructed by replacing the first hidden layer of a baseline DNN with stochastic hidden layer.
We first train a baseline DNN for first 200 epochs, and
the trained parameters of DNN are used for initializing those of Simplified-SFNN.
For 50 epochs, we train simplified-SFNN.
We choose the hyper-parameter $\gamma_2=50$ in the parameter transformation.
All Simplified-SFNNs are trained with $M=20$ samples at each epoch, and in the test, we use 500 samples. 
Table \ref{Tab:relu_SSFNN} shows that SFNN under the two-stage training always performs better than
SFNN under a simple transformation \eqref{eq:relutransform} from ReLU-DNN.
More interestingly,
Simplified-SFNN consistently outperforms its baseline DNN due to the stochastic regularizing effect, even
when we train both models using
dropout \citep{12dropout} and batch normalization \citep{15batch}.
This implies that the proposed stochastic model can be used for improving the performance of DNNs and it can be also combined with other regularization methods such as dropout batch normalization.
In order to confirm the regularization effects, 
one can again approximate a trained Simplified-SFNN by a new deterministic DNN which we call DNN$^*$
and is different from its baseline DNN under the following approximation at upper latent
units above binary random units: 
\begin{align} \label{eq:approximation2}
\E_{P\left(\mathbf{h}^{\ell}\mid \mathbf{x}\right)} \left[s\left(\mathbf{W}^{\ell+1}_{j}\mathbf{h}^{\ell}  \right)\right]  \backsimeq  s\left(\E_{P\left(\mathbf{h}^{\ell}\mid \mathbf{x}  \right)}\left[\mathbf{W}^{\ell+1}_{j}\mathbf{h}^{\ell} \right]\right) 
 =s\left( \sum_{i} W^{\ell+1}_{ij}P\left(h^\ell_i=1\mid \mathbf{x} \right) \right). 
\end{align}
We found that DNN$^*$ using fined-tuned parameters of Simplified-SFNN
also outperforms the baseline DNN 
as shown in Table \ref{Tab:relu_SSFNN} and Figure \ref{fig:fig1a}.

\subsection{Training Simplified-SFNN} \label{sec:train}
The parameters of Simplified-SFNN can be learned 
using a variant of the backpropagation algorithm \citep{85BackProp} 
in a similar manner to DNN. However, in contrast to DNN,  
there are two computational issues for simplified-SFNN: 
computing expectations with respect to stochastic units in forward pass and computing gradients in back pass. 
One can notice that both are intractable since they require summations over all possible configurations of all stochastic units.
First, in order to handle the issue in forward pass, we use the following
Monte Carlo approximation for estimating the expectation:
\begin{align*}
\E_{P\left(\mathbf{h}^{1}\mid \mathbf{x}\right)} \left[s\left(\mathbf{W}^{2}_{j}\mathbf{h}^{1}+b_j^{2}\right)\right]
\backsimeq \frac{1}{M} \sum \limits_{m=1}^M s\left(\mathbf{W}^{2}_{j}\mathbf{h}^{(m)}+b_j^{2}\right), 
\end{align*}
where $\mathbf{h}^{(m)} \sim P\left(\mathbf{h}^{1}\mid \mathbf{x}\right)$ and $M$ is the number of samples.
This random estimator is unbiased and has relatively low variance \citep{13SFNN} since its accuracy 
does not depend on the dimensionality of $\mathbf{h}^{1}$ and one can draw samples from the exact distribution.
Next, in order to handle the issue in back pass, we use the following approximation inspired by \citep{15SFNN}:
\begin{align*}
\frac{\partial}{\partial \mathbf{W}^{2}_{j}} \E_{P\left(\mathbf{h}^{1}\mid \mathbf{x}\right)} \left[s\left(\mathbf{W}^{2}_{j}\mathbf{h}^{1}+b_j^{2}\right)\right] 
& \backsimeq \frac{1}{M} \sum\limits_{m} \frac{\partial}{\partial \mathbf{W}^{2}_{j}} s\left(\mathbf{W}^{2}_{j}\mathbf{h}^{(m)} +b_j^{2}\right), \\
 \frac{\partial}{\partial \mathbf{W}^{1}_{i}} \E_{P\left(\mathbf{h}^{1}\mid \mathbf{x}\right)} \left[s\left(\mathbf{W}^{2}_{j}\mathbf{h}^{1}+b_j^{2}\right)\right] 
&  \backsimeq  \frac{W^{2}_{ij}}{M} \sum\limits_{m}  s^{\prime}\left(\mathbf{W}^{2}_{j}\mathbf{h}^{(m)} +b_j^{2}\right) \frac{\partial}{\partial \mathbf{W}^{1}_{i}} P\left(h_i^{1}=1\mid \mathbf{x}\right),
\end{align*}
where $\mathbf{h}^{(m)} \sim P\left(\mathbf{h}^{1}\mid \mathbf{x}\right)$ and $M$ is the number of samples.
In our experiments, we commonly choose $M=20$.
\section{Extensions of Simplified-SFNN} \label{sec:mainresult}
In this section, we describe how 
the network knowledge transferring
between Simplified-SFNN and DNN, i.e.,
Theorem \ref{thm:2samefeature}, generalizes
to multiple layers and general activation functions.

\subsection{Extension to multiple layers} \label{sec:CompVC}
A deeper Simplified-SFNN with $L$ hidden layers can be defined similarly as the case of $L=2$.
We also establish network knowledge transferring between Simplified-SFNN and DNN with $L$ hidden layers as stated in the following theorem.
Here, we assume that
stochastic layers are not consecutive for simpler presentation, but
the theorem is generalizable for consecutive stochastic layers.


\begin{theorem} \label{thm:Lsamefeature}
Assume that both DNN and Simplified-SFNN with $L$ hidden layers have same network structure with non-negative 
activation function $f$.
Given parameters $\{ \mathbf{\widehat W}^\ell,\mathbf{\widehat b}^\ell: { \ell =1,\dots, L} \}$ of DNN 
and input dataset $D$, 
choose the same ones
for Simplified-SFNN initially and modify them for each $\ell$-th stochastic layer and its upper layer
as follows:
\begin{align}
\alpha_\ell
\leftarrow   \frac{1}{\gamma_\ell}, 
\left( \alpha_{\ell+1}, \mathbf{W}^{\ell+1},~\mathbf{b}^{\ell+1} \right)
\leftarrow  
\left(  \frac{\gamma_{\ell}\gamma_{\ell+1}}{ s^{\prime} \left(0\right)}, ~
\frac{\mathbf{\widehat{W}}^{\ell+1}}{\gamma_{\ell+1}} ,~
\frac{\mathbf{\widehat{b}}^{\ell+1}}{\gamma_{\ell}\gamma_{\ell+1} }   \right),
\label{eq:Lparasetup} 
\end{align}
where $\gamma_\ell= \max \limits_{i,\mathbf{x}\in D} ~\left| f\left(\mathbf{\widehat W}^{\ell}_{i} \mathbf{h}^{\ell-1}(\mathbf{x})  +{\widehat b}_i^{\ell}\right)\right|$ and $\gamma_{\ell+1}$ is any positive constant.
Then, it follows that
\begin{align*}
\lim \limits_{\substack{\gamma_{\ell+1} \to \infty \\ \forall~\text{stochastic hidden layer}~\ell}} \left| h_j^L\left(\mathbf{x}\right) 
-\widehat{h}_j^L\left(\mathbf{x}\right) \right|  = 0,
\quad \forall j,\mathbf{x}\in D.
\end{align*}
\end{theorem}
The above theorem again implies that it is possible to transfer knowledge from DNN to Simplified-SFNN by choosing large $\gamma_{l+1}$.
The proof of Theorem \ref{thm:Lsamefeature} is
similar to that of Theorem \ref{thm:2samefeature} and given
in Section \ref{sec:appendix}.

\subsection{Extension to general activation functions} \label{sec:General}
In this section, we describe an extended version of Simplified-SFNN which can utilize any activation function.
To this end, we modify the definitions of stochastic layers and their upper layers
by introducing certain additional terms.
If the $\ell$-th hidden layer is stochastic, 
then we slightly modify the original definition \eqref{def:pro} as follows: 
\begin{align*}
\label{def:extpro} \nonumber
P\left(\mathbf{h}^\ell\mid\mathbf{x}\right) = \prod \limits_{i=1}^{N^\ell} P\left(h_i^\ell \mid \mathbf{x}\right) 
\text{with} ~~
P\left(h_i^{\ell}=1 \mid \mathbf{x}\right) = \min\bigg\{\alpha_{\ell} f\left(\mathbf{W}^{1}_{i} \mathbf{x}+b_i^{1} + \frac{1}{2}\right) ,1\bigg\},
\end{align*}
where $f:\mathbb{R}\rightarrow \mathbb{R}$ is a non-linear (possibly, negative)
activation function with $\left| f^\prime(x) \right| \leq 1$ for all $x\in \mathbb R$.
In addition, we re-define its upper layer as follows:
\begin{align*}
 \mathbf{h}^{\ell+1}\left(\mathbf{x}\right) =\bigg[f \Big( \alpha_{\ell+1} \Big(   \E_{P\left(\mathbf{h}^{\ell}\mid \mathbf{x}\right)} \left[s\left(\mathbf{W}^{\ell+1}_{j}\mathbf{h}^{\ell}+b_j^{\ell+1}\right)\right]  
  - s\left(0\right) {-\frac{s^{\prime}\left(0\right)}{2} \sum_i W_{ij}^{\ell+1}} \Big) \Big):\forall j  \bigg], 
\end{align*}
where $\mathbf{h}^0 (\mathbf{x})=\mathbf{x}$ and
$s:\mathbb{R}\rightarrow \mathbb{R}$ is a differentiable function with 
$\left|s^{\prime\prime}(x)\right|\leq 1$ for all $x\in \mathbb R$. 

Under this general Simplified-SFNN model, we also show that transferring network knowledge from DNN to Simplified-SFNN is possible as stated in the following theorem.
Here, we again assume that
stochastic layers are not consecutive for simpler presentation.

\begin{theorem} \label{thm:LGsamefeature}
Assume that both DNN and Simplified-SFNN with $L$ hidden layers have same network structure with non-linear activation function $f$.
Given parameters $\{ \mathbf{\widehat W}^\ell,\mathbf{\widehat b}^\ell: { \ell =1,\dots, L} \}$ of DNN 
and input dataset $D$, 
choose the same ones for Simplified-SFNN initially and modify them for each $\ell$-th stochastic layer and its upper layer as follows:
\begin{align*}
\alpha_\ell
\leftarrow   \frac{1}{2\gamma_\ell}, 
\left( \alpha_{\ell+1}, \mathbf{W}^{\ell+1},~\mathbf{b}^{\ell+1} \right)
\leftarrow  
\left(  \frac{2\gamma_{\ell}\gamma_{\ell+1}}{ s^{\prime} (0)}, ~
\frac{\mathbf{\widehat{W}}^{\ell+1}}{\gamma_{\ell+1}} ,~
\frac{\mathbf{\widehat{b}}^{\ell+1} }{2\gamma_{\ell}\gamma_{\ell+1} }  \right),
\end{align*}
where $\gamma_\ell= \underset{i,\mathbf{x}\in D}{{\max}} ~\left|f\left(\mathbf{\widehat W}^{\ell}_{i} \mathbf{h}^{\ell-1}(\mathbf{x})  +{\widehat b}_i^{\ell}\right)\right|$,
and $\gamma_{\ell+1}$ is any positive constant. 
Then, 
it follows that
\begin{align*}
\lim \limits_{\substack{\gamma_{\ell+1} \to \infty \\ \forall~\text{stochastic hidden layer}~ \ell}} \left| h_j^L\left(\mathbf{x}\right) 
-\widehat{h}_j^L\left(\mathbf{x}\right) \right|  = 0,
\quad \forall j,\mathbf{x}\in D.
\end{align*}
\end{theorem}
We omit the proof of the above theorem
since it is somewhat
direct adaptation of that of Theorem \ref{thm:Lsamefeature}.

\section{Proofs of Theorems}
\subsection{Proof of Theorem \ref{thm:2samefeature}} \label{sec:proofVCSig}
First consider the first hidden layer, i.e., stochastic layer. 
Let $\gamma_1 = \max \limits_{i,\mathbf{x}\in D} f\left( \mathbf{\widehat W}^{1}_{i} \mathbf{x}+{\widehat b}_i^{1}\right) $ 
be the maximum value of hidden units in DNN.
If we initialize the parameters $\left( \alpha_1, \mathbf{W}^{1},~\mathbf{b}^{1} \right)
\leftarrow  \left( \frac{1}{\gamma_1},~
\mathbf{\widehat{W}}^{1},~
\mathbf{\widehat{b}}^{1}   \right)$,
then the marginal distribution of each hidden unit $i$ becomes
\begin{align}
P\left(h_i^1=1\mid\mathbf{x}, \mathbf{W}^1,\mathbf{b}^1\right)  
= min\bigg\{ \alpha_1 f\left(\mathbf{\widehat W}^1_{i} \mathbf{x}+{\widehat b}_i^1\right)  ,1\bigg\}  
= \frac{1}{\gamma_1} f\left(\mathbf{\widehat W}^1_{i} \mathbf{x}+{\widehat b}_i^1\right) , ~ \forall i, \mathbf{x}\in D.
\label{eq:firstlayer}
\end{align}
Next consider the second hidden layer.
From Taylor's theorem, there exists a value $z$ between $0$ and $x$ such that 
$s(x) 
= s(0) + s^{\prime}(0)x  +R(x)$, where $R(x) = \frac{ s^{\prime\prime} (z) x^2}{2!}  $.
Since we consider a binary random vector, i.e., $\mathbf{h}^1 \in \{0,1 \}^{N^1}$,
one can write
\begin{align}
&\E_{P\left(\mathbf{h}^1\mid \mathbf{x}\right)} \left[s\left(\beta_j\left(\mathbf{h}^1\right)\right)\right] \nonumber \\ 
& = \sum \limits_{\mathbf{h}^1} \left( s\left(0\right) + s^{\prime}\left(0\right)\beta_j\left(\mathbf{h}^1\right)
+ R\left(\beta_j\left(\mathbf{h}^1\right) \right) \right) P\left(\mathbf{h}^1\mid \mathbf{x}\right) \nonumber \\ 
&= s\left(0\right)
+s^{\prime}\left(0\right) \left(  \sum \limits_{i} W^2_{ij}  P(h_i^1=1\mid \mathbf{x}) + b_j^2  \right) 
+ \E_{P\left(\mathbf{h}^1\mid\mathbf{x}\right)} \left[ R(\beta_j(\mathbf{h}^1)) \right], 
\end{align}
where $ \beta_j\left(\mathbf{h}^1\right) := \mathbf{W}^2_{j}\mathbf{h}^1+b_j^2$ is the incoming signal. 
From \eqref{def:deter} and \eqref{eq:firstlayer}, for every hidden unit $j$, it follows that 
\begin{align*}
h_j^2\left(\mathbf{x}; \mathbf{W}^2,\mathbf{b}^2\right)  
= &f \Bigg( \alpha_2  \Bigg( s'(0) \left( \frac{1}{\gamma_1} \sum \limits_{i} W^2_{ij} {\widehat h}_i^1\left(\mathbf{x}\right)
+b_j^2 \right) 
 + \E_{P\left(\mathbf{h}^1\mid \mathbf{x}\right)}\left[ R\left(\beta_j\left(\mathbf{h}^1\right)\right)\right] \Bigg) \Bigg).
\end{align*}
Since we assume that $\left| f'(x) \right| \leq 1$, the following inequality holds:
\begin{align}
&\left\lvert h_j^2(\mathbf{x}; \mathbf{W}^2,\mathbf{b}^2) - f \left( \alpha_2  s'(0) \left( \frac{1}{\gamma_1} \sum \limits_{i} W^2_{ij} {\widehat h}_i^1(\mathbf{x})
+b_j^2 \right)    \right)  \right\rvert  \notag \\
&\leq  \left| \alpha_2  \E_{P\left(\mathbf{h}^1\mid \mathbf{x}\right)}\left[ R(\beta_j(\mathbf{h}^1))\right] \right| 
\leq \frac{\alpha_2}{2}
\E_{P(\mathbf{h}^1\mid\mathbf{x})} \left[\left(  \mathbf{W}^2_{j}\mathbf{h}^1+b_j^2 \right)^2 \right],
\end{align}
where we use 
$|s^{\prime\prime}(z)|<1$ for the last inequality.
Therefore, it follows that 
\begin{align*}
\left| h_j^2\left(\mathbf{x}; \mathbf{W}^2,\mathbf{b}^2\right) 
-\widehat{h}_j^2\left(\mathbf{x} ; \mathbf{\widehat W}^2,\mathbf{\widehat b}^2\right) \right| 
\leq 
\frac{\gamma_1\left(\sum_i \left| {\widehat W}_{ij}^2\right|+ {\widehat b}_{j}^2\gamma_1^{-1} \right)^2}{2s' (0)\gamma_2}, \quad \forall j,
\end{align*}
since we set $\left( \alpha_2, \mathbf{W}^{2},~\mathbf{b}^{2} \right)
\leftarrow  
\left(  \frac{\gamma_2\gamma_1}{ s^{\prime} (0)}, ~
\frac{ \mathbf{\widehat{W}}^{2}}{\gamma_2},~
\frac{\gamma_1^{-1}}{\gamma_2 } \mathbf{\widehat{b}}^{2}  \right)$. 
This completes the proof of Theorem \ref{thm:2samefeature}.

\subsection{Proof of Theorem \ref{thm:Lsamefeature}} \label{sec:appendix}

For the proof of Theorem \ref{thm:Lsamefeature},
we first state the two key lemmas on error propagation in Simplified-SFNN.
\begin{lemma} \label{le:errorbound}
Assume that there exists some positive constant $B$ such that 
\begin{align*}
\left| h_i^{\ell-1} \left(\mathbf{x}\right) 
-\widehat{h}_i^{\ell-1} \left(\mathbf{x}\right) \right| 
 \leq B, \quad \forall i, \mathbf{x} \in D,
\end{align*}
and the $\ell$-th hidden layer of Simplified-SFNN is standard deterministic layer as defined in \eqref{ea:ENN}.
Given parameters $\{ \mathbf{\widehat W}^\ell,\mathbf{\widehat b}^\ell \}$ of DNN, choose same ones for Simplified-SFNN.
Then, the following inequality holds:
\begin{align*}
\left| h_j^{\ell} \left(\mathbf{x}\right) 
-\widehat{h}_j^{\ell} \left(\mathbf{x}\right) \right| 
 \leq B N^{\ell-1}{\widehat W}_{\max}^\ell, \quad \forall j, \mathbf{x} \in D.
\end{align*}
where ${\widehat W}_{\max}^\ell= \max\limits_{ij} \left|{\widehat W}_{ij}^\ell\right|$.
\end{lemma}
\begin{proof}
See Section \ref{sec:appendix2}.
\end{proof}
\begin{lemma} \label{le:errorbound2}
Assume that there exists some positive constant $B$ such that 
\begin{align*}
\left| h_i^{\ell-1} \left(\mathbf{x}\right) 
-\widehat{h}_i^{\ell-1} \left(\mathbf{x}\right) \right| 
 \leq B, \quad \forall i, \mathbf{x} \in D,
\end{align*}
and the $\ell$-th hidden layer of simplified-SFNN is stochastic layer.
Given parameters $\{ \mathbf{\widehat W}^{\ell},\mathbf{\widehat W}^{\ell+1},\mathbf{\widehat b}^{\ell},\mathbf{\widehat b}^{\ell+1} \}$ of DNN, choose those of Simplified-SFNN as follows:
\begin{align*}
\alpha_\ell
\leftarrow   \frac{1}{\gamma_\ell}, 
 \left( \alpha_{\ell+1}, \mathbf{W}^{\ell+1},~\mathbf{b}^{\ell+1} \right)
\leftarrow  
\left(  \frac{\gamma_{\ell}\gamma_{\ell+1}}{ s^{\prime} \left(0\right)}, ~
\frac{\mathbf{\widehat{W}}^{\ell+1}}{\gamma_{\ell+1}} ,~
\frac{\mathbf{\widehat{b}}^{\ell+1}}{\gamma_{\ell}\gamma_{\ell+1} }   \right), 
\end{align*}
where $\gamma_\ell= \max \limits_{j,\mathbf{x}\in D} ~\left| f\left(\mathbf{\widehat W}^{\ell}_{j} \mathbf{h}^{\ell-1}(\mathbf{x})  +{\widehat b}_j^{\ell}\right)\right|$ and $\gamma_{\ell+1}$ is any positive constant.
Then, for all $j,\mathbf{x}\in D$, it follows that 
\begin{align*}
\left| h_k^{\ell+1}\left(\mathbf{x}\right) 
-\widehat{h}_k^{\ell+1} \left(\mathbf{x}\right) \right| 
\leq & BN^{\ell-1}N^{\ell}{\widehat W}_{\max}^\ell {\widehat W}_{\max}^{\ell+1} 
 + \left|\frac{\gamma_{\ell}\left( N^{\ell}{\widehat W}_{\max}^{\ell+1} + {\widehat b}_{\max}^{\ell+1} \gamma_{\ell}^{-1} \right)^2}{2s' (0)\gamma_{\ell+1}}   \right|,
\end{align*}
where
${\widehat b}_{\max}^{\ell} = \max \limits_j
 \left|{\widehat b}_j^{\ell}\right|$ and
${\widehat W}_{\max}^\ell= \max\limits_{ij} \left|{\widehat W}_{ij}^\ell\right|$.
\end{lemma}
\begin{proof}
See Section \ref{sec:appendix3}.
\end{proof}

Assume that $\ell$-th layer is first stochastic hidden layer in Simplified-SFNN. Then, from Theorem \ref{thm:2samefeature}, we have  
\begin{align}
 \left| h_j^{\ell+1}\left(\mathbf{x}\right) 
-\widehat{h}_j^{\ell+1} \left(\mathbf{x}\right) \right| 
\leq \left|\frac{\gamma_{\ell}\left( N^{\ell}{\widehat W}_{\max}^{\ell+1} + {\widehat b}_{\max}^{\ell+1} \gamma_{\ell}^{-1} \right)^2}{2s' (0)\gamma_{\ell+1}}   \right|,\quad \forall j, \mathbf{x} \in D.
\label{eq:firsterror}
\end{align}
According to Lemma \ref{le:errorbound} and \ref{le:errorbound2}, the final error generated by the right hand side of \eqref{eq:firsterror} is bounded by
\begin{align}
\frac{\tau_\ell \gamma_{\ell}\left(N^{\ell}{\widehat W}_{\max}^{\ell+1}  + {\widehat b}_{\max}^{\ell+1}\gamma_{\ell}^{-1}\right)^{2}}{2s^{\prime} \left(0\right)\gamma_{\ell+1}},
\label{eq:errorbobobobob}
\end{align}
where $\tau_\ell = \prod \limits_{\ell'=l+2}^L \left( N^{\ell'-1}{\widehat W}_{\max}^{\ell'}\right).$ One can note that every error generated by each stochastic layer is bounded by \eqref{eq:errorbobobobob}.
Therefore, for all $j,\mathbf{x}\in D$, it follows that 
\begin{align*}
\left| h_j^L\left(\mathbf{x}\right) 
-\widehat{h}_j^L\left(\mathbf{x}\right) \right|  
 \leq \sum \limits_{\ell :\text{stochastic hidden layer}}\left( 
\frac{\tau_\ell \gamma_{\ell}\left(N^{\ell}{\widehat W}_{\max}^{\ell+1}  + {\widehat b}_{\max}^{\ell+1}\gamma_{\ell}^{-1}\right)^{2}}{2s^{\prime} \left(0\right)\gamma_{\ell+1}} \right). 
\end{align*}
From above inequality, we can conclude that 
\begin{align*}
\lim \limits_{\substack{\gamma_{\ell+1} \to \infty \\ \forall~\text{stochastic hidden layer}~\ell}} \left| h_j^L\left(\mathbf{x}\right) 
-\widehat{h}_j^L\left(\mathbf{x}\right) \right|  = 0,
\quad \forall j,\mathbf{x}\in D.
\end{align*}
This completes the proof of Theorem \ref{thm:Lsamefeature}.

\subsection{Proof of Lemma \ref{le:errorbound}} \label{sec:appendix2}

From assumption, there exists some constant $\epsilon_i$ such that $\left| \epsilon_i \right|<B$ and 
\begin{align*} h_i^{\ell-1} \left(\mathbf{x}\right)  =\widehat{h}_i^{\ell-1} \left(\mathbf{x}\right)+\epsilon_i, ~~ \forall i, \mathbf{x}.\end{align*}
By definition of standard deterministic layer, it follows that
 \begin{align*}
h_j^{\ell}\left(\mathbf{x} \right) 
=f \left(  \sum \limits_{i} {\widehat W}^{\ell}_{ij} {h}_{i}^{\ell-1}\left(\mathbf{x} \right)
+{\widehat b}_j^{\ell-1}   \right) 
 = f \left(  \sum \limits_{i} {\widehat W}^{\ell}_{ij} {\widehat h}_{i}^{\ell-1}\left(\mathbf{x} \right) + \sum \limits_{i} {\widehat W}^{\ell}_{ij} \epsilon_i
+{\widehat b}_j^{\ell}   \right) .
\end{align*}
Since we assume that $\left| f'(x) \right| \leq 1$, one can conclude that
\begin{align*}
\left| h_j^{\ell}\left(\mathbf{x} \right) -  f \left(  \sum \limits_{i} {\widehat W}^{\ell}_{ij} {\widehat h}_{i}^{\ell-1}\left(\mathbf{x} \right) 
+{\widehat b}_j^{\ell}   \right)\right| 
\leq \left| \sum \limits_{i} {\widehat W}^{\ell}_{ij} \epsilon_i \right|  
 \leq B \left| \sum \limits_{i}  {\widehat W}^{\ell}_{ij} \right| 
\leq B N^{\ell-1}{\widehat W}_{\max}^\ell.
\end{align*}
This completes the proof of Lemma \ref{le:errorbound}.

\subsection{Proof of Lemma \ref{le:errorbound2}} \label{sec:appendix3}

From assumption, there exists some constant $\epsilon^{\ell-1}_i$ such that $\left| \epsilon^{\ell-1}_i \right|<B$ and 
\begin{align}
 h_i^{\ell-1} \left(\mathbf{x}\right)  =\widehat{h}_i^{\ell-1} \left(\mathbf{x}\right)+\epsilon^{\ell-1}_i, ~~ \forall i, \mathbf{x}.
 \label{eq:assump12}
\end{align}
Let $\gamma_\ell = \max \limits_{j,\mathbf{x}\in D} ~\left| f\left(\mathbf{\widehat W}^{\ell}_{j} \mathbf{h}^{\ell-1}(\mathbf{x})  +{\widehat b}_j^{\ell}\right)\right|$ be the maximum value of hidden units.
If we initialize the parameters $\left( \alpha_\ell, \mathbf{W}^{\ell},~\mathbf{b}^{\ell} \right)
\leftarrow  \left( \frac{1}{\gamma_\ell},~
\mathbf{\widehat{W}}^{\ell},~
\mathbf{\widehat{b}}^{\ell}   \right)$,
then the marginal distribution becomes
\begin{align*}
 P\left(h_j^{\ell}=1\mid\mathbf{x}, \mathbf{W}^{\ell},\mathbf{b}^{\ell}\right)  &=  \min \bigg\{ \alpha_\ell f\left(\mathbf{\widehat W}^{\ell}_{j}\mathbf{ h}^{\ell-1}\left( \mathbf{x}\right)+{\widehat b}_j^\ell\right)  ,1\bigg\} \\
&=\frac{1}{\gamma_{\ell}} f\left( \mathbf{\widehat W}^{\ell}_{j}\mathbf{ h}^{\ell-1}\left( \mathbf{x}\right) +{\widehat b}_j^{\ell}\right) ,~ \forall j, \mathbf{x}.
\end{align*}
From \eqref{eq:assump12}, it follows that 
\begin{align*}
P\left(h_j^{\ell}=1\mid \mathbf{x}, \mathbf{W}^{\ell},\mathbf{b}^{\ell}\right)  
=
\frac{1}{\gamma_{\ell}} f\left( \mathbf{\widehat W}^{\ell}_{j}\mathbf{\widehat h}^{\ell-1}\left( \mathbf{x}\right)+ \sum \limits_{i} {\widehat W}^{\ell}_{ij} \epsilon^{\ell-1}_i +{\widehat b}_j^{\ell}\right) , \quad \forall j, \mathbf{x} .
\end{align*}
Similar to Lemma \ref{le:errorbound}, there exists some constant $\epsilon^{\ell}_j$ such that $\left| \epsilon^{\ell}_j \right|<BN^{\ell-1}{\widehat W}_{\max}^\ell$ and 
\begin{align}
 P\left(h_j^{\ell}=1\mid \mathbf{x}, \mathbf{W}^{\ell},\mathbf{b}^{\ell}\right)   = \frac{1}{\gamma_\ell} \left( \widehat{h}_j^{\ell} 
\left(\mathbf{x}\right)+\epsilon^{\ell}_j \right), ~~ \forall j, \mathbf{x}.
\label{eq:eqeq11}
\end{align}
Next, consider the upper hidden layer of stochastic layer.
From Taylor's theorem, there exists a value $z$ between $0$ and $t$ such that 
$s(x) 
= s(0) + s'(0)x  +R(x)$, where $R(x) = \frac{ s'' (z) x^{2}}{2!}  $.
Since we consider a binary random vector, i.e., $\mathbf{h}^{\ell} \in \{0,1 \}^{N^{\ell}}$,
one can write
\begin{align*}
\E_{P(\mathbf{h}^{\ell}\mid \mathbf{x})} [s(\beta_k(\mathbf{h}^{\ell}))] 
& = \sum \limits_{\mathbf{h}^{\ell}} \left( s(0) + s'(0)\beta_k(\mathbf{h}^{\ell})
+ R\left(\beta_k(\mathbf{h}^{\ell}) \right) \right) P(\mathbf{h}^{\ell}\mid \mathbf{x}) \nonumber \\ 
&= s(0)
+s'(0) \left(  \sum \limits_{j} W^{\ell+1}_{jk}  P(h_j^\ell=1\mid \mathbf{x}) + b_k^{\ell+1}  \right) 
 + \sum \limits_{\mathbf{h}^{\ell}} R(\beta_k(\mathbf{h}^{\ell}))  P(\mathbf{h}^{\ell}\mid \mathbf{x}), \nonumber
\end{align*}
where $ \beta_k(\mathbf{h}^{\ell}) = \mathbf{W}^{\ell+1}_{k}\mathbf{h}^{\ell}+b_k^{\ell+1}$ is the incoming signal. 
From \eqref{eq:eqeq11} and above equation, for every hidden unit $k$, we have
\begin{align*}
&h_k^{\ell+1}(\mathbf{x}; \mathbf{W}^{\ell+1},\mathbf{b}^{\ell+1}) \\
&=f \Bigg( \alpha_{\ell+1}  \Bigg( s'(0) \Bigg( \frac{1}{\gamma_\ell} \left( \sum \limits_{j} W^{\ell+1}_{jk} {\widehat h}_j^{\ell}(\mathbf{x})
+ \sum \limits_{j} W^{\ell+1}_{jk} \epsilon^{\ell}_j \right) 
 +b_k^{\ell+1} \Bigg)
+ \E_{P(\mathbf{h}^{\ell}\mid\mathbf{x})}\left[ R(\beta_k(\mathbf{h}^{\ell}))\right]     \Bigg) \Bigg).
\end{align*}
Since we assume that $|f'(x)|<1$, the following inequality holds:
\begin{align}
&\Bigg| h_k^{\ell+1}(\mathbf{x};\mathbf{W}^{\ell+1},\mathbf{b}^{\ell+1}) 
-f \left( \alpha_{\ell+1}  s^{\prime}(0) \left( \frac{1}{\gamma_\ell} \sum \limits_{j} W^{\ell+1}_{ij} {\widehat h}_j^{\ell}(\mathbf{x})
+b_j^{\ell+1} \right)    \right) \Bigg| \nonumber \\ 
&\leq \left|
 \frac{ \alpha_{\ell+1}  s'(0) }{\gamma_\ell} \sum \limits_{j} W^{\ell+1}_{jk} \epsilon^{\ell}_j 
 +  \alpha_{\ell+1} \E_{P(\mathbf{h}^{\ell}\mid\mathbf{x})}\left[ R(\beta_k(\mathbf{h}^{\ell}))\right]  \right| \nonumber \\
 &\leq \left|
 \frac{ \alpha_{\ell+1}  s'(0) }{\gamma_\ell} \sum \limits_{j} W^{\ell+1}_{jk} \epsilon^{\ell}_j \right|
  + \left| \frac{\alpha_{\ell+1}}{2} \E_{P(\mathbf{h}^{\ell}\mid \mathbf{x})}\left[\left( \mathbf{W}^{\ell+1}_{k}\mathbf{h}^{\ell}+b_k^{\ell+1}\right)^2 \right]  \right|,
\label{eq:gohome}
\end{align}
where we use 
$|s^{\prime\prime}(z)|<1$ for the last inequality.
Therefore, it follows that 
\begin{align*}
 \left| h_k^{\ell+1}\left(\mathbf{x}\right) 
-\widehat{h}_k^{\ell+1} \left(\mathbf{x}\right) \right| \leq & ~ BN^{\ell-1}N^{\ell}{\widehat W}_{\max}^\ell {\widehat W}_{\max}^{\ell+1} 
 + \left|\frac{\gamma_{\ell}\left( N^{\ell}{\widehat W}_{\max}^{\ell+1} + {\widehat b}_{\max}^{\ell+1} \gamma_{\ell}^{-1} \right)^2}{2s' (0)\gamma_{\ell+1}}   \right|,
\end{align*}
since we set $\left( \alpha_{\ell+1}, \mathbf{W}^{\ell+1},~\mathbf{b}^{\ell+1} \right)
\leftarrow  
\left(  \frac{\gamma_{\ell+1}\gamma_{\ell}}{ s^{\prime} (0)}, ~
\frac{\mathbf{\widehat{W}}^{\ell+1}}{\gamma_{\ell+1}} ,~
\frac{\gamma_{\ell}^{-1}\mathbf{\widehat{b}}^{\ell+1}}{\gamma_{\ell+1} }   \right)$. 
This completes the proof of Lemma \ref{le:errorbound2}.

\section{Experimental Results} \label{sec:uaiexp}
\begin{figure*} [t!] \centering
\subfigure[]
{\epsfig{file=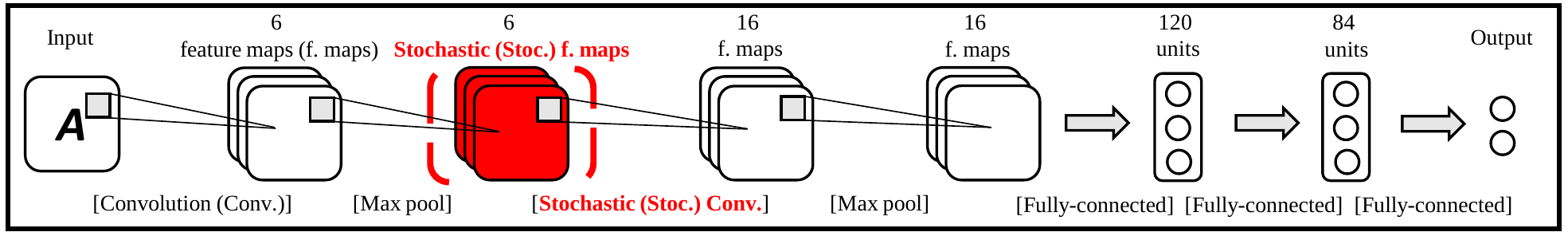, width=1\textwidth}\label{fig:fig3a}}
\,
\subfigure[]
{\epsfig{file=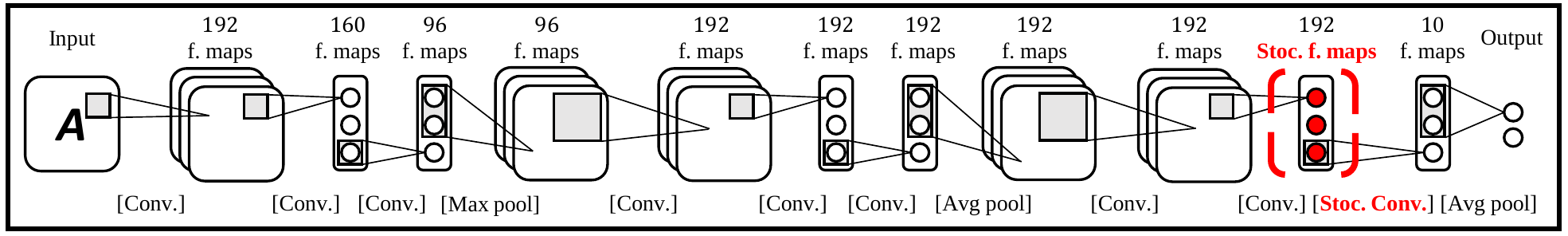, width=1\textwidth}\label{fig:fig3b}}
\,
\subfigure[]
{\epsfig{file=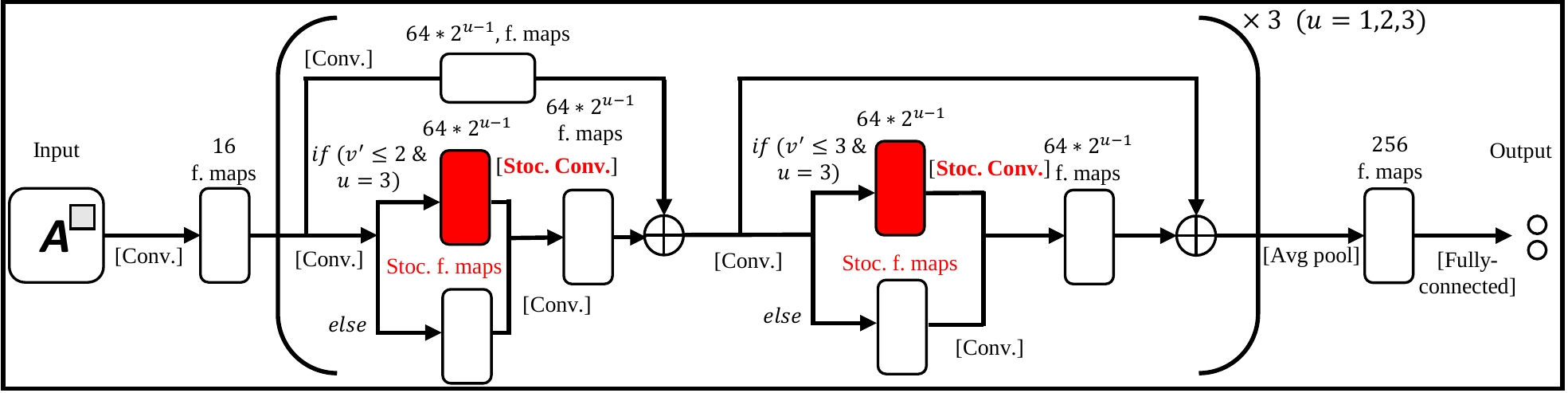, width=1\textwidth}\label{fig:fig3c}}
\,
\subfigure[]
{\epsfig{file=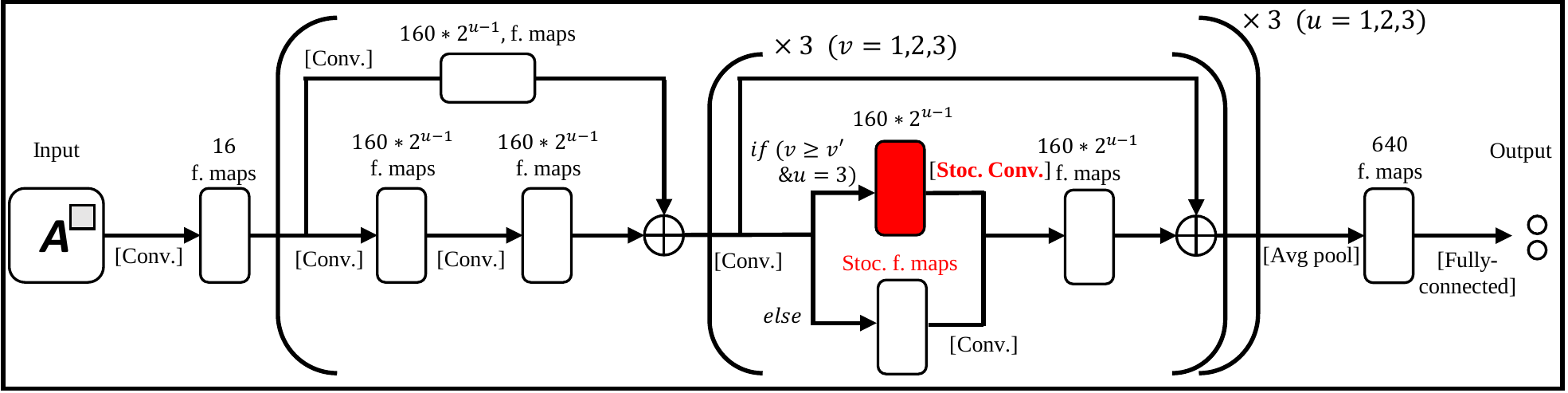, width=1\textwidth}\label{fig:fig3c}}
\caption{The overall structures of (a) Lenet-5, (b) NIN, (c) WRN with 16 layers, and (d) WRN with 28 layers. The red feature maps correspond to the stochastic ones. In case of WRN, we introduce one ($v^{\prime}=3$) and two ($v^{\prime}=2$) stochastic feature maps.}
\end{figure*} 

We present several experimental results for both multi-modal and classification tasks on
MNIST \citep{98MNIST}, Toronto Face Database (TFD) \citep{10TFD}, 
CASIA \citep{liu2011casia},
CIFAR-10, CIFAR-100 \citep{09CIFAR} and SVHN \citep{11SVHN}.
The softmax and Gaussian with the standard deviation of 0.05 are used as the output probability for the classification task and the multi-modal prediction, respectively.
In all experiments, we first train a baseline model, and the trained parameters are used for further fine-tuning those of Simplified-SFNN.

\subsection{Multi-modal regression}

We first verify that it is possible to learn one-to-many mapping via Simplified-SFNN on the TFD, MNIST and CASIA datasets.
The TFD dataset consists of $48\times48$ pixel greyscale images, 
each containing a face image of 900 individuals with 7 different expressions.
Similar to \citep{15SFNN}, we use 124 individuals with at least 10 facial expressions as data. 
We randomly choose 100 individuals with 1403 images for training
and the remaining 24 individuals with 326 images for the test.
We take the mean of face images per individual as the input and
set the output as the different expressions of the same individual.
The MNIST dataset consists of greyscale images, 
each containing a digit 0 to 9 with 60,000 training and 10,000 test images. 
For this experiments, each pixel of every digit images is binarized using its grey-scale value similar to \citep{15SFNN}.
We take the upper half of the MNIST digit as the input
and set the output as the lower half of it.
We remark that both tasks are commonly performed in recent other works to test the multi-modal learning using SFNN \citep{15SFNN,15muprop}.
The CASIA dataset consists of greyscale images,
each containing a handwritten Chinese character.
We use 10 offline isolated characters produced by 345 writers as data.\footnote{We use 5 radical characters (e.g., wei, shi, shui, ji and mu in Mandarin) and 5 related characters (e.g., hui, si, ben, ren and qiu in Mandarin).}
We randomly choose 300 writers with 3000 images for training
and the remaining 45 writers with 450 images for testing.
We take the radical character per writer as the input and 
set the output as either the related character or the radical character itself of the same writer (see Figure \ref{CASIA:gen}).
The bicubic interpolation \citep{keys1981cubic} is used for re-sizing all images as $32\times32$ pixels.

\begin{figure*} [t!] \centering
\subfigure[]
{\epsfig{file=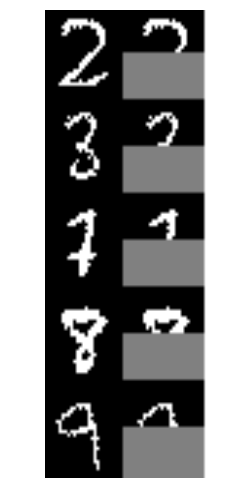,
width=0.1\textwidth}\label{fig:mnist1}}
\,
\subfigure[]
{\epsfig{file=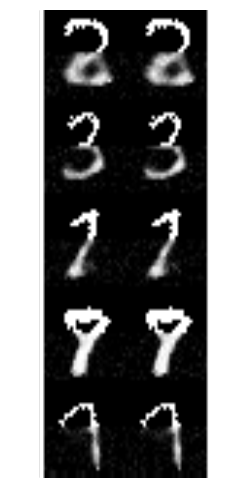,
width=0.1\textwidth}\label{fig:mnist2}}
\,
\subfigure[]
{\epsfig{file=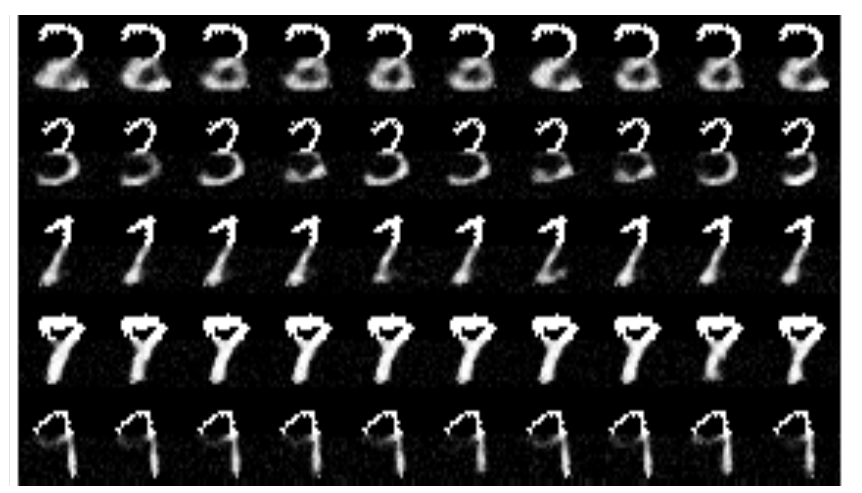,
width=0.35\textwidth}\label{fig:mnist3}}
\,
\subfigure[]
{\epsfig{file=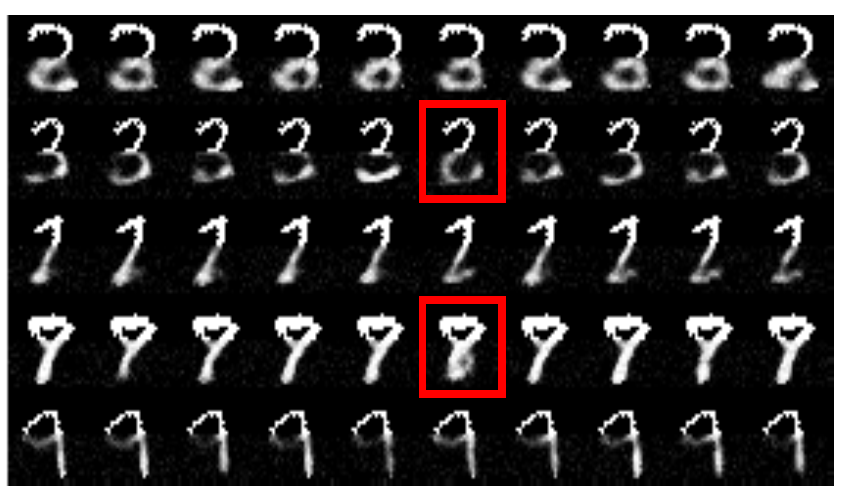,
width=0.35\textwidth}\label{fig:mnist4}}
\caption{Generated samples for predicting the lower half of the MNIST digit given the upper half. (a) The original digits and the corresponding inputs.
The generated samples from (b) sigmoid-DNN, (c) SFNN under the simple transformation, and 
(d) SFNN fine-tuned by Simplified-SFNN.
We remark that 
SFNN fine-tuned by Simplified-SFNN can generate more various samples (e.g., red rectangles) from same inputs better than SFNN under the simple transformation.}
\label{fig:genearion}
\end{figure*}

For both TFD and MNIST datasets, we use fully-connected DNNs as the baseline models similar to other works \citep{15SFNN,15muprop}.
All Simplified-SFNNs are constructed by replacing the first hidden layer of a baseline DNN with stochastic hidden layer. 
The loss was minimized using ADAM learning rule \citep{15ADAM} with a mini-batch size of 128. We used an exponentially decaying learning rate.
We train Simplified-SFNNs with $M=20$ samples at each epoch, and in the test, we use 500 samples.
We use 200 hidden units for each layer of neural networks in two experiments. 
Learning rate is chosen from \{0.005 , 0.002, 0.001, ... , 0.0001\}, and the best result is reported for both tasks.
Table \ref{Tab:mult_exp2} reports the test negative log-likelihood on TFD and MNIST.
One can note that SFNNs fine-tuned by Simplified-SFNN consistently outperform SFNN under the simple transformation.

For the CASIA dataset, we choose
fully-convolutional network (FCN) models \citep{long2015fully} as the baseline ones, 
which consists of convolutional layers followed by a fully-convolutional layer. 
In a similar manner to the case of fully-connected networks,
one can define a stochastic convolution layer, which considers 
the input feature map as a binary random matrix and generates the output feature map as defined in \eqref{def:deter}.
For this experiment, 
we use a baseline model, which consists of convolutional layers followed by a fully convolutional layer. 
The convolutional layers have 64, 128 and 256 filters respectively. 
Each convolutional layer has a4×4receptive field applied with a stride of 2 pixel.
All Simplified-SFNNs are constructed by replacing the first hidden feature maps of baseline models with stochastic ones.
The loss was minimized using ADAM learning rule \citep{15ADAM} with a mini-batch size of 128. We used an exponentially decaying learning rate.
We train Simplified-SFNNs with $M=20$ samples at each epoch, and in the test, we use 100 samples due to the memory limitations.
One can note that
SFNNs fine-tuned by Simplified-SFNN outperform SFNN under the simple transformation as reported in Table \ref{Tab:casia} and Figure \ref{CASIA:gen}.

\begin{table}[t!] 
\centering
\caption{Test NLL on CASIA dataset. All Simplified-SFNNs are constructed by replacing the first hidden feature maps of baseline models with stochastic ones.} 
\label{Tab:casia}
\begin{tabular}{@{}ccclclclll@{}}
\toprule
\multirow{1}{*}{\begin{tabular}[c]{@{}c@{}} Inference Model \end{tabular}} & \multirow{1}{*}{\begin{tabular}[c]{@{}c@{}}      Training Model \end{tabular}} 
& \multicolumn{1}{c}{\begin{tabular}[c]{@{}c@{}} 2 Conv Layers \end{tabular}}& \multicolumn{1}{c}{\begin{tabular}[c]{@{}c@{}} 3 Conv Layers \end{tabular}}\\ \midrule
\multicolumn{1}{c}{FCN}
& \multicolumn{1}{c}{FCN} & \multicolumn{1}{c}{3.89} & \multicolumn{1}{c}{3.47}\\
\multicolumn{1}{c}{SFNN}
 & \multicolumn{1}{c}{FCN} &  \multicolumn{1}{c}{2.61}& \multicolumn{1}{c}{1.97} \\
 \multicolumn{1}{c}{\begin{tabular}[c]{@{}c@{}} Simplified-SFNN \end{tabular}}
& \multicolumn{1}{c}{\begin{tabular}[]{@{}c@{}}fine-tuned by Simplified-SFNN\end{tabular}} &  \multicolumn{1}{c}{3.83} & \multicolumn{1}{c}{3.47}\\
\multicolumn{1}{c}{\begin{tabular}[c]{@{}c@{}} SFNN \end{tabular}}
& \multicolumn{1}{c}{\begin{tabular}[]{@{}c@{}}fine-tuned by Simplified-SFNN\end{tabular}} &  \multicolumn{1}{c}{2.45}& \multicolumn{1}{c}{\bf 1.81} \\ \bottomrule
\end{tabular}
\end{table}

\begin{figure*} [t!]
\CenterFloatBoxes
\begin{floatrow}
\ffigbox{
\subfigure[]
{\epsfig{file=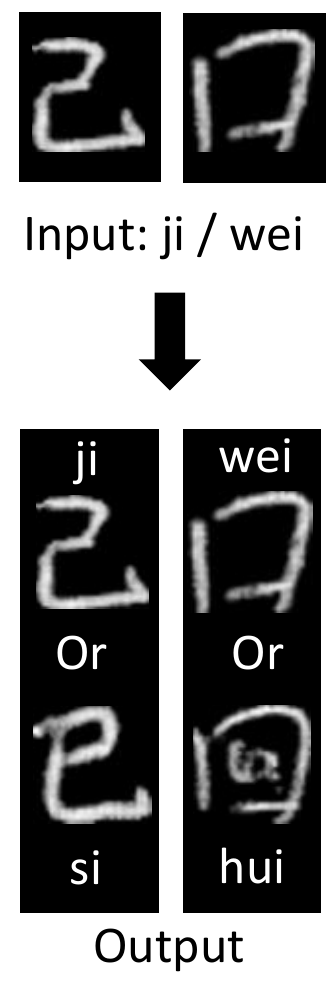,
width=0.1\textwidth}\label{fig:casia1}}
\,
\subfigure[]
{\epsfig{file=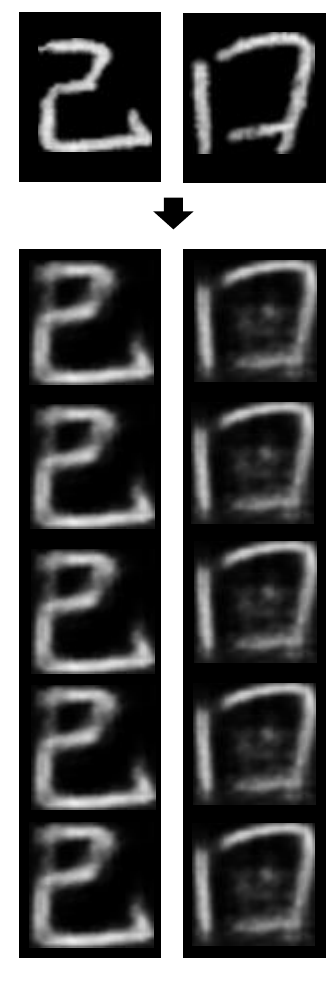,
width=0.1\textwidth}\label{fig:casia2}}
\,
\subfigure[]
{\epsfig{file=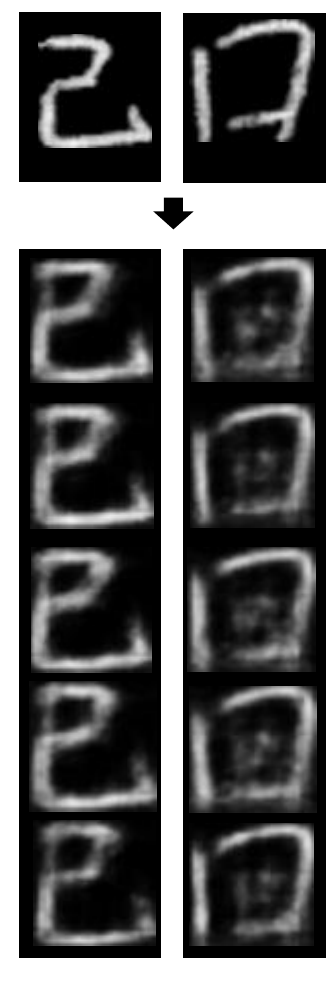,
width=0.1\textwidth}\label{fig:casia3}}
\,
\subfigure[]
{\epsfig{file=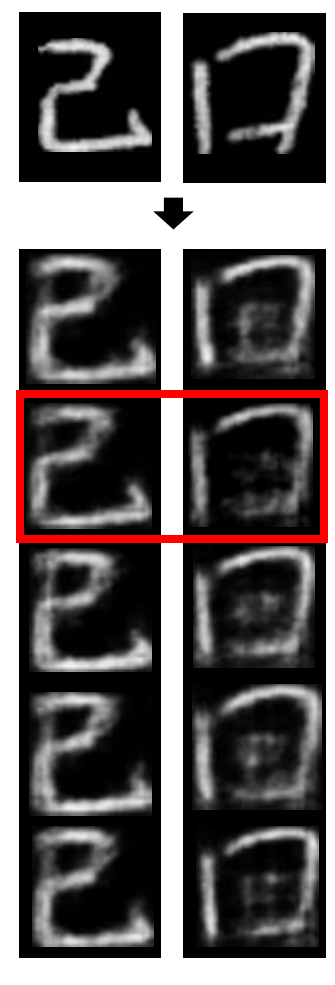,
width=0.1\textwidth}\label{fig:casia4}}}
{\caption{
(a) The input-output Chinese characters.
The generated samples from (b) FCN, (c) SFNN under the simple transformation and (d) SFNN fine-tuned by Simplified-SFNN.
Our model can generate the multi-modal outputs (red rectangle),
while SFNN under the simple transformation cannot.}\label{CASIA:gen}}
\ffigbox{
\vspace{-0.5in}
\includegraphics [width=0.44\textwidth]{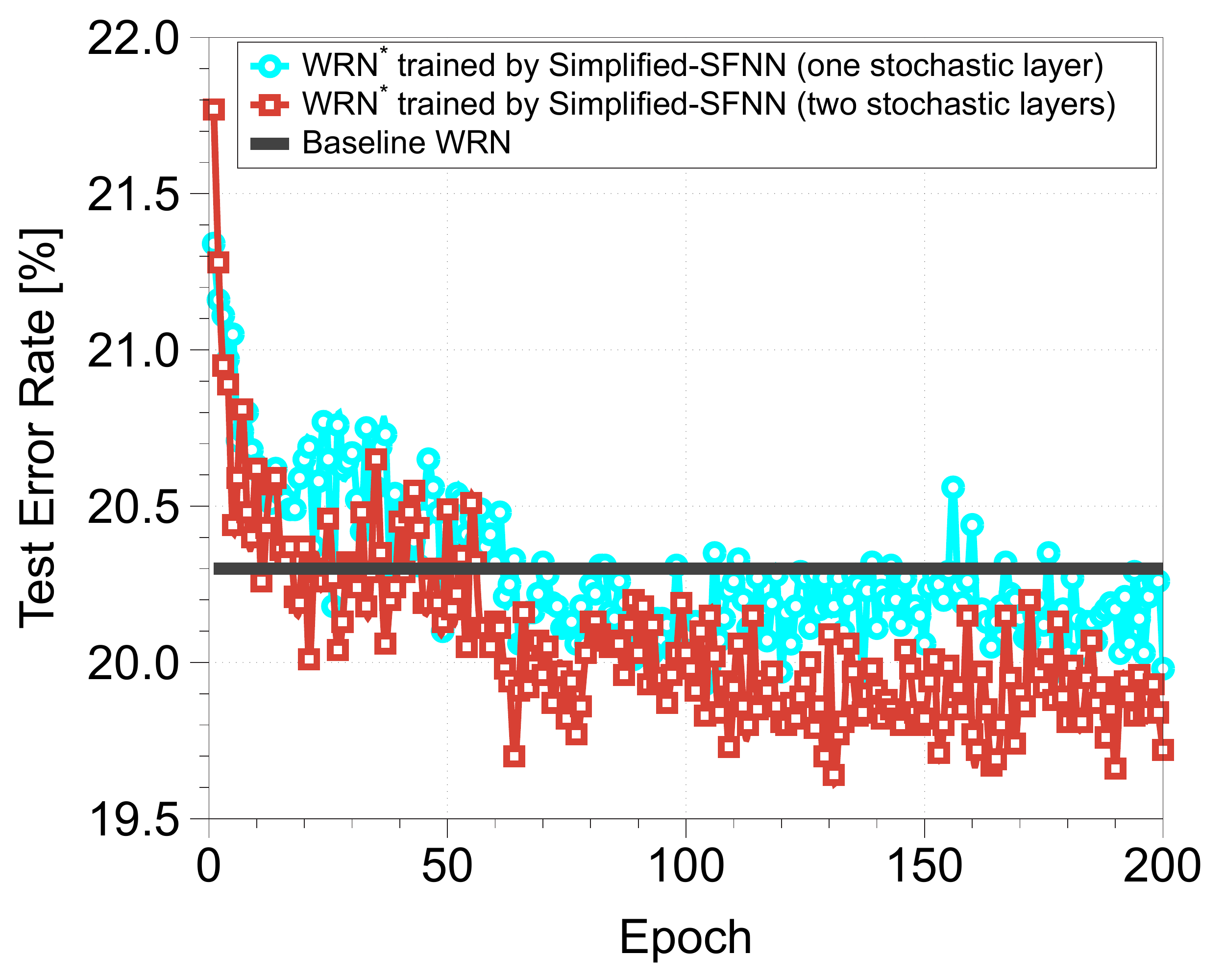}}
{\caption{Test error rates of WRN$^*$ per each training epoch on CIFAR-100. One can note that the performance gain is more significant when we introduce more stochastic layers.}\label{wrn_training}}
\killfloatstyle
\end{floatrow}
\end{figure*}

\subsection{Classification} 

We also evaluate the regularization effect of Simplified-SFNN for the classification tasks on CIFAR-10, CIFAR-100 and SVHN.
The CIFAR-10 and CIFAR-100 datasets consist of 50,000 training and 10,000 test images.
They have 10 and 100 image classes, respectively. 
The SVHN dataset consists of 73,257 training
and 26,032 test images.\footnote{We do not use the extra SVHN dataset for training.}
It consists of a house number 0 to 9 collected by Google Street View. 
Similar to \citep{16wideresnet}, we pre-process the data using global contrast normalization and ZCA whitening. 
For these datasets, 
we design a convolutional version of Simplified-SFNN using convolutional neural networks such as Lenet-5 \citep{98MNIST}, NIN \citep{14NIN} and WRN \citep{16wideresnet}.
All Simplified-SFNNs are constructed by replacing a hidden feature map of a baseline models, i.e., Lenet-5, NIN and WRN, with stochastic one as shown in Figure \ref{fig:fig3c}.
We use WRN with 16 and 28 layers for SVHN and CIFAR datasets, respectively,
since they showed state-of-the-art performance as reported by \citep{16wideresnet}.
In case of WRN,
we introduce up to two stochastic convolution layers.
Similar to \citep{16wideresnet}, the loss was minimized using the stochastic gradient descent method with Nesterov momentum.
The minibatch size is set to 128, and weight decay is set to 0.0005.
For 100 epochs, we first train baseline models, i.e., Lenet-5, NIN and WRN,
and trained parameters are used for initializing those of Simplified-SFNNs.
All Simplified-SFNNs are trained with $M=5$ samples and the test error is only measured by the approximation \eqref{eq:approximation2}.
The test errors of baseline models are measured after training them for 200 epochs similar to \citep{16wideresnet}.
All models are trained 
using dropout 
\citep{12dropout} 
and batch normalization
\citep{15batch}.
Table \ref{Tab:CIFAR2} reports the classification error rates on CIFAR-10, CIFAR-100 and SVHN.
Due to the regularization effects, Simplified-SFNNs consistently outperform their baseline DNNs.
In particular, WRN$^*$ 
of 28 layers and 36 million parameters 
outperforms WRN by 0.08$\%$ on CIFAR-10 and 0.58$\%$ on CIFAR-100.
Figure \ref{wrn_training} shows that the error rate is decreased more
when we introduce more stochastic layers, 
but it increases the fine-tuning time-complexity of Simplified-SFNN.

\begin{table*}[t!] 
\centering
\caption{Test error rates [$\%$] on CIFAR-10, CIFAR-100 and SVHN. The error rates for WRN are from our experiments, where original ones reported in \citep{16wideresnet} are in the brackets. Results with $\dagger$ are obtained using the horizontal flipping and random cropping augmentation.} 
\label{Tab:CIFAR2}
\resizebox{\textwidth}{!}{
\begin{tabular}{@{}llllll@{}}
\toprule
\begin{tabular}{c}Inference \\ model \end{tabular} & Training Model  & \multicolumn{1}{c}{CIFAR-10} & \multicolumn{1}{c}{CIFAR-100}  & \multicolumn{1}{c}{SVHN}    \\ \midrule
\begin{tabular}{c}Lenet-5 \end{tabular} & \begin{tabular}{c}Lenet-5\end{tabular}  &\multicolumn{1}{c}{37.67}  &\multicolumn{1}{c}{77.26} &\multicolumn{1}{c}{11.18}\\
\begin{tabular}{c}Lenet-5$^*$ \end{tabular}& \begin{tabular}{c}fine-tuned by Simplified-SFNN \end{tabular}  & \multicolumn{1}{c}{33.58}  & \multicolumn{1}{c}{73.02} & \multicolumn{1}{c}{9.88} \\ \midrule
\begin{tabular}{c}NIN \end{tabular} &\begin{tabular}{c}NIN\end{tabular}  & \multicolumn{1}{c}{9.51} & \multicolumn{1}{c}{32.66} & \multicolumn{1}{c}{3.21} \\
\begin{tabular}{c}NIN$^*$ \end{tabular} &\begin{tabular}{c}fine-tuned by Simplified-SFNN \end{tabular}   & \multicolumn{1}{c}{9.33} & \multicolumn{1}{c}{30.81} & \multicolumn{1}{c}{3.01} \\ \midrule
\begin{tabular}{c}WRN \end{tabular}& \begin{tabular}{c}WRN \end{tabular}  &\multicolumn{1}{c}{4.22 (4.39)$\dagger$}  &\multicolumn{1}{c}{20.30 (20.04)$\dagger$} &\multicolumn{1}{c}{3.25$\dagger$} \\ 
\begin{tabular}{c}WRN$^*$ \end{tabular} & 
\begin{tabular}{l}
fine-tuned by Simplified-SFNN   (one stochastic layer)  \end{tabular} & \multicolumn{1}{c}{4.21$\dagger$}  & \multicolumn{1}{c}{19.98$\dagger$}&\multicolumn{1}{c}{3.09$\dagger$} \\ 
\begin{tabular}{c}WRN$^*$ \end{tabular} & \begin{tabular}{l}
fine-tuned by Simplified-SFNN   (two stochastic layers)  \end{tabular}   & \multicolumn{1}{c}{{\bf 4.14}$\dagger$}  & \multicolumn{1}{c}{{\bf 19.72}$\dagger$}&\multicolumn{1}{c}{3.06$\dagger$} \\  \bottomrule
\end{tabular}}
\end{table*}

\section{Conclusion} \label{sec:conc}

In order to develop an efficient training method for large-scale SFNN, 
this paper proposes a new intermediate stochastic model, called Simplified-SFNN.
We establish the connection between three models, i.e., DNN $\rightarrow$ Simplified-SFNN $\rightarrow$ SFNN, which naturally leads to an efficient training procedure of the stochastic models utilizing pre-trained parameters of DNN.
This connection naturally leads an efficient training procedure of the stochastic models utilizing pre-trained parameters and architectures of DNN.
Using several popular DNNs including
Lenet-5, NIN, FCN and WRN, we show how they can be effectively transferred to the corresponding stochastic models for both multi-modal and non-multi-modal tasks.
We believe that our work brings a new angle for training stochastic neural networks, which
would be of broader interest in many related applications.

{
\begingroup
\subsubsection*{References}
\renewcommand{\section}[2]{}

\endgroup}

\end{document}